\documentclass{ecai} 

\usepackage{latexsym}
\usepackage{amssymb}
\usepackage{amsmath}
\usepackage{amsthm}
\usepackage{booktabs}
\usepackage{enumitem}
\usepackage{graphicx}
\usepackage{color}
\usepackage{multirow}
\usepackage{subfig}
\usepackage{hyperref}
\usepackage[capitalize]{cleveref}

\newtheorem{theorem}{Theorem}

\newcommand{\q}[1]{``#1''}
\newcommand{\BibTeX}{B\kern-.05em{\sc i\kern-.025em b}\kern-.08em\TeX}

\begin{document}


\begin{frontmatter}


\paperid{123} 


\title{Molecular Topological Profile (MOLTOP) - Simple and Strong Baseline for Molecular Graph Classification}


\author[A]{\fnms{Jakub}~\snm{Adamczyk}\orcid{0000-0003-4336-4288}\thanks{Corresponding Author. Email: jadamczy@agh.edu.pl.}}
\author[B]{\fnms{Wojciech}~\snm{Czech}\orcid{0000-0002-1903-8098}}

\address[A]{Faculty of Computer Science, AGH University of Krakow, Krakow, Poland}
\address[B]{Faculty of Computer Science, AGH University of Krakow, Krakow, Poland}

\begin{abstract}
We revisit the effectiveness of topological descriptors for molecular graph classification and design a simple, yet strong baseline. We demonstrate that a simple approach to feature engineering - employing histogram aggregation of edge descriptors and one-hot encoding for atomic numbers and bond types - when combined with a Random Forest classifier, can establish a strong baseline for Graph Neural Networks (GNNs). The novel algorithm, Molecular Topological Profile (MOLTOP), integrates Edge Betweenness Centrality, Adjusted Rand Index and SCAN Structural Similarity score. This approach proves to be remarkably competitive when compared to modern GNNs, while also being simple, fast, low-variance and hyperparameter-free. Our approach is rigorously tested on MoleculeNet datasets using fair evaluation protocol provided by Open Graph Benchmark. We additionally show out-of-domain generation capabilities on peptide classification task from Long Range Graph Benchmark. The evaluations across eleven benchmark datasets reveal MOLTOP's strong discriminative capabilities, surpassing the $1$-WL test and even $3$-WL test for some classes of graphs. Our conclusion is that descriptor-based baselines, such as the one we propose, are still crucial for accurately assessing advancements in the GNN domain.
\end{abstract}

\end{frontmatter}


\section{Introduction}

Graph classification has become a crucial type of supervised learning problem, increasingly relevant across various scientific domains. This surge in importance is largely attributed to the expanding quantity of structured datasets that represent pairwise relationships among various types of modeled entities. Graph classification algorithms are utilized in a variety of fields, particularly in chemoinformatics, where their application in Quantitative Structure-Activity Relationship (QSAR) modeling plays a critical role in predicting the functions of biochemically significant molecules \cite{deng2023systematic}. Particularly, the prediction of ADME (Absorption, Distribution, Metabolism, Excretion) pharmacokinetic properties plays a pivotal role in supporting contemporary in-silico drug design \cite{fang2023prospective}. 

Graph classification confronts a fundamental difficulty: measuring the dissimilarity between objects that are not situated in a metric space. Therefore, graphs, unlike more straightforward tabular or categorical data, require special methods to capture their complexity and relationships. Traditionally, this problem was solved by extracting isomorphism-invariant representations of graphs in the form of feature vectors, also known as graph embeddings, descriptors, or fingerprints. Alternatively, explicit pairwise similarity measures, known as graph kernels, can be constructed to systematically compare graph substructures \cite{kriege2020survey}. Both methods remain intrinsically unsupervised or task-independent, however domain-specific knowledge can be incorporated by careful feature engineering.
Although graph descriptors have achieved success in various benchmark classification tasks, more recently, they are often surpassed by the more advanced graph representation learning models exemplified by Graph Neural Networks (GNNs). They learn task-specific representations and can take advantage of pre-training to reduce negative effects of limited training data \cite{Pretraining_GNNs_Hu,R_MAT_Maziarka}. In developing a universal framework for graph classification, GNN models frequently incorporate descriptors, either as a method of input data augmentation or as supplementary global features in the readout layer \cite{DMPNN_Yang, dwivedi2023benchmarking}.

Given the computational expense of graph representation learning, the requirement for extensive training data, the challenge of transferring pre-trained knowledge to specialized prediction tasks, and the prevalence of domain-specific graph descriptors, comparing GNNs with traditional methods remains valuable. This is particularly true when descriptor-based methods serve as a baseline, indicating whether GNNs can learn additional, task-specific features. 
The studies \cite{liao2021we} and \cite{dehghani2021benchmark} have identified significant obstacles hindering progress in the field of machine learning. These include challenges in effectively evaluating models, particularly issues related to non-replicable results and comparisons using inadequate baselines. Besides, the study presented in work \cite{Statistical_comparisons_Demsar} advocates for statistical rigor, when comparing classifiers across multiple datasets.
More specifically, in the graph classification field, the authors of \cite{GNN_fair_comparison_Errica} describe problems with replicating GNN results caused by lack of strict separation between model selection and model evaluation step. Moreover, they show that under a fair comparison framework, simple structure-agnostic baselines can outperform GNN models such as GIN or GraphSAGE. In \cite{luzhnica2019graph} the authors demonstrate that trivial $1$-layer GCN can perform on par with complex GNNs such as DiffPool. The work \cite{zhang2022nafs} similarly notes the effectiveness of training-free vertex descriptors in link prediction tasks.
In the realm of molecular graph classification it was shown that descriptor-based models, particularly those utilizing molecular fingerprints, not only yield better average prediction results than GNN models but also are computationally cheaper by an order of magnitude \cite{pappu2020making,jiang2021could}. The clear need of comparable prediction results and maintaining fair leaderboards led to the creation of benchmark datasets and related evaluation protocols such OGB \cite{OGB_Hu}, MoleculeNet \cite{MoleculeNet_Wu} or TDC \cite{huang2021therapeutics}.

Motivated by research underscoring the value of robust baselines, and inspired by recent methods utilizing graph topology descriptors \cite{LDP,LTP}, we propose Molecular Topological Profile (MOLTOP), a baseline method for molecular graph classification, utilizing both topological descriptors and simple atom and bond features. The resulting baseline, under the fair evaluation protocols offered by modern benchmarks, results in a surprisingly efficient and strong model, able to outperform contemporary GNNs. Our method is fast, scalable, robust in distinguishing graphs, non-parametric, and it exhibits low-variance in prediction tasks. Additionally, we present the studies verifying expressive power and feature importance of the proposed representation.

The code is available at \href{https://github.com/j-adamczyk/MOLTOP}{https://github.com/j-adamczyk/MOLTOP}.

\section{Related works}
\label{related_works}

Graph descriptors, which generate isomorphism-invariant vectors representing graphs, exemplify the feature-engineering approach to graph classification. Descriptors are versatile in representing features at different levels – from granular to aggregated, local to global \cite{czech2012invariants}, and from purely structural aspects to those including multidimensional labels \cite{lazarz2021relation}.  
In practical applications, the descriptors from spectral graph theory \cite{de2018simple, pineau2019using} or the ones using histogram aggregation of vertex/edge topological features \cite{LDP, LTP} have successfully rivaled more complex methods. The approach of generic graph descriptors was expanded by incorporating domain-specific representations, like molecular fingerprints, which have become widely used in predicting biochemical properties and molecular database search. Typical fingerprints are bit-vectors of a given size, built based on depth-first search explorations from each atom, and incorporating its 2D \cite{durant2002reoptimization,stiefl2006erg} or 3D structure \cite{axen2017simple, nilakantan1987topological}. The molecular property prediction based on molecular fingerprints can be highly competitive to GNNs, as shown in \cite{jiang2021could} and evident from OGB leaderboards \cite{OGB_Hu}.

Bypassing the need for manual feature engineering, GNNs provide an automated method for extracting task-specific graph features and transporting them directly to a trainable readout layer. Starting from early works introducing Graph Convolutional Network (GCN) \cite{GCN} and GraphSAGE \cite{hamilton2017inductive} the field of graph representation learning has evolved significantly, leading to the development of numerous models, as categorized by \cite{liu2022taxonomy}. Some of these models, e.g., Graph Isomorphism Networks (GIN) \cite{GIN} achieved state-of-the-art performance in benchmark graph classification tasks including molecular property prediction. GIN was designed to match the discriminative power of the Weisfeiler-Lehman isomorphism test, thereby offering additional insights into the representational capabilities of GNNs. Subsequently, in \cite{DMPNN_Yang} the authors proposed a hybrid model D-MPNN, which combines edge-centered graph convolutions and molecular descriptors concatenated at the readout layer. That work represents a significant advancement in molecular graph classification, notable not only for its comprehensive and detailed analysis of model efficiency but also for its successful integration of the strengths of both GNNs and descriptors. In their work, \cite{AttentiveFP_Xiong} adopted the graph attention mechanism to develop the AttentiveFP model. This method is capable of utilizing atom and bond features, effectively extracting both local and global properties of a molecule. 

When operating in a low-resource learning regime, GNNs often struggle to build discriminative representations of graphs. The success of transfer learning in the field of Natural Language Processing (NLP), coupled with the scarcity of training data in molecular property prediction, has inspired researchers to adopt different pre-training strategies tailored for GNNs.
In their work, \cite{Pretraining_GNNs_Hu} introduced a comprehensive framework that employs pre-training techniques like context prediction or attribute masking. This approach enables the transfer of knowledge from large molecular datasets to general-purpose GNNs, enhancing classification accuracy on benchmark tasks.
In parallel, the transformer-style architecture GROVER was introduced by \cite{GROVER_Rong}, reporting notable advancements over existing state-of-the-art methods. It utilized the largest pre-training database at the time, comprising 10 million unlabeled molecules. Graph Contrastive Learning (GraphCL) \cite{GraphCL} was introduced as another self-supervised learning method, leveraging parameterized graph augmentations and maximizing the mutual information between augmentations sharing the same semantics. GraphCL was further extended by enabling automatic, adaptive selection of augmentation parameters \cite{JOAO} (JOAO). New pre-training strategies leveraging 2D topological structures extracted by encoders and enriched by 3D views led to development efficient GraphMVP framework \cite{liu2021pre}. More recently, the GEM model \cite{GEM_Fang} proposed incorporating 3D molecular properties, based on Merck molecular force field (MMFF) simulations. Combining those geometric features with the GIN model and pretraining on 20 million molecules from the ZINC database led to exceptional performance in a range of graph classification and regression tasks, although at a very high computational cost. The most recent work, \cite{R_MAT_Maziarka} describes relative molecule self-attention transformer (R-MAT), which uses atom embeddings reflecting relative distance between atoms. R-MAT reports SOTA results of molecular benchmarks, but uses different datasets and data splits than other models, therefore it is difficult to compare to this approach.

In contrast to multiple works reporting high efficiency of pre-trained GNN models, many thorough ablation studies, such as \cite{GNN_pretraining_benchmark_Sun}, provide contrary results. They present important findings on why feature engineering combined with low parameter machine learning can still outperform complex models, and why the pre-training benefits can be diminished in practical property prediction setups.

\section{Preliminaries}
\label{preliminaries}

\textbf{Molecular graph.} Let $G = (V, E)$ denote an undirected graph representing a molecule, where $V$ and $E$ are the sets of vertices (nodes, atoms) and edges (links, bonds), respectively. We also mark $G = (A, X_n, X_e)$, where $A$ is the adjacency matrix, $X_n$ is the node feature matrix and $X_e$ is the edge feature matrix.

\textbf{Graph notation}. We denote single vertices as $v$ or $u$, and edges as two element vertex sets $e = \{u, v\}$. $\mathcal{N}(v)$ is the set of neighbors of a node $v$, $\mathrm{deg}(v)$ is the degree of a node $v$, i.e. the number of its neighbors, $\mathrm{deg}(v) = |\mathcal{N}(v)|$.

\textbf{Graph classification.} We consider the graph classification task, where we are given a dataset $\mathbb{D} = \left( G^{(i)}, Y^{(i)} \right)$, $i = 1, 2, ..., N$, of $N$graphs and their labels. Class (label) for a given graph $Y^{(i)}$ is a boolean for single-task datasets. For multitask datasets, it is a binary vector of length $T$ (for $T$ binary classification tasks), and it can have missing labels.

\section{Method}
\label{methods}

We propose Molecular Topological Profile (MOLTOP) as a baseline method for benchmarking against GNNs in molecular graph classification tasks. Baselines are simple and computationally cheap methods, expected to provide a reference point for more sophisticated methods. While not a focus point of any paper, they are a necessary part of fair valuation of new algorithms, especially on new datasets.

For MOLTOP, given its role as a baseline method, simplicity and speed are just as crucial as classification accuracy. In order to achieve good performance on chemical data, we utilize both topological and molecular features. The method relies on extracting feature vectors from graphs independently, and using Random Forest to classify resulting tabular data. In contrast to previous baselines, either purely topological (e.g. LDP \cite{LDP} and LTP \cite{LTP}), or purely feature-based (\q{molecular fingerprint} from \cite{GNN_fair_comparison_Errica}), it incorporates both graph structure and atoms and bonds features, all of which are crucial in chemistry.

The first group of features we consider are vertex degree statistics, to directly summarize the basic topology of a $2$-hop neighborhood around each node \cite{LDP}. We denote the multiset of vertex neighbors' degrees as $DN(v) = \{ \mathrm{deg}(u) | u \in \mathcal{N}(v) \}$. For each atom, we then calculate the following statistics: $\mathrm{deg}(v)$, $\min(DN(v))$, $\max(DN(v))$, $\mathrm{mean}(DN(v))$, $\mathrm{std}(DN(v))$. In order to create graph-level features, they are compactly represented using histograms, a technique akin to the global readout in GNNs, but with higher expressivity than just simple mean or sum \cite{Weave_Kearnes}.

For molecular graphs, especially in medicinal chemistry, having a degree higher than $8$ is very rare. Using the same number of bins for all features would result in a very large number of all-zero features for many molecules. Therefore, we propose to reduce the number of bins to $11$ for $\mathrm{deg}(v)$, $\mathrm{min}(DN(v))$ and $\mathrm{max}(DN(v))$. This covers singular hydrogens, covalent bonds, and nearly all atoms with higher degrees than $8$ (e.g. due to ionic or metallic bonding) in typical biochemistry.

Inspired by previous structural approaches and path-based molecular fingerprints, we add further topological descriptors to enhance this representation. We select features that work well for describing molecular fragments, and that should discriminate well between different scaffolds and functional groups. Concretely, we selected Edge Betweenness Centrality (EBC), Adjusted Rand Index (ARI) and SCAN Structural Similarity score. Each of those descriptors is computed for edges (bonds), but focuses on a different aspect of molecule structure. EBC considers global graph connectivity structure and its shortest path-based properties. ARI uses $3$-hop subgraphs and neighborhood connectivity patterns. SCAN also considers local connectivity patterns, but is based on the notion of node clusters and outliers. Therefore, those features should provide complementary information in the feature vector. Another reason for utilizing edge features is that similar edge-focused approaches were successful in improving GNNs for molecular property prediction \cite{DMPNN_Yang,Weave_Kearnes,GPS_transformer}. Each of those features is calculated for all bonds in the molecule and aggregated with a histogram.

\textbf{Edge betweenness centrality (EBC)} \cite{EBC_Girvan} is a centrality measure for edges, defined as a fraction of shortest paths going through that edge:
\begin{equation}
EBC(e) = \frac{2}{|V|(|V|-1)} \sum_{u, v \in V} \frac{\sigma_{u,v}(e)}{\sigma_{u,v}},
\end{equation}
where $\sigma_{u,v}$ is the total number of shortest paths between $u$ and $v$, and $\sigma_{u,v}(e)$ is the number of that paths going through $e$. The normalization factor before the sum ensures that the values lie in range $[0, 1]$ and are unaffected by graph size. Information about the shortest paths in the graph is well known to be important in chemistry, being used e.g. in Wiener index and Hyper-Wiener index \cite{Wiener_type_indices}, and was also successfully incorporated into multiple GNNs \cite{Graphormer_Ying,GPS_transformer}. However, the histogram of centralities includes more information than the lengths of shortest paths would, because it shows the actual distribution of critically important edges. If there are bonds with very high EBC values, it indicates the existence of bridge-like subgraphs, such as glycosidic bonds. It can also easily distinguish between linear and polycyclic scaffolds, since the ring-rich topologies will have smaller EBC values in general, while linear structures have many high-EBC bonds.

\textbf{Adjusted Rand Index (ARI)} \cite{ARI_Hoffman} is a normalized measure of overlap between neighborhoods of two vertices $u$ and $v$:
$$
ARI(u, v) = \frac{2(ad - bc)}{(a + b)(b + d) + (a + c)(c + d)},
$$
where $a$ is the number of edges to other vertices that $u$ and $v$ have in common, $b$ is the number of edges to other nodes for $u$ that $v$ does not have, $c$ is the number of edges to other nodes for $v$ that $u$ does not have, and $d$ is the number of edges to other nodes that neither $u$ nor $v$ has. Calculated for edge $e = \{u, v\}$, it provides information about the subgraphs of radius 3 (from neighbors of $v$, through edge $e$, to neighbors of $u$).

Among various neighborhood overlap measures, ARI has a particularly strong statistical interpretation, being equivalent to Kohen's $\kappa$ defined on incident edge sets of $u$ and $v$ \cite{ARI_Hoffman}. While this measure is typically used for link prediction, it can also be calculated for existing edges. This method has been used for identifying 'incorrect' links, where it surpassed other techniques \cite{ARI_Hoffman}, and a similar approach was also used in LTP \cite{LTP}. Therefore, the histogram of ARI values should work well for existing edges, taking into consideration larger subgraphs than degree features and indicating the general connectivity patterns in a graph. In particular, it is capable of differentiating between star-like graphs (such as spiro compounds or functional groups containing atoms with high coordination number, e.g. phosphate groups) and polycyclic molecules characterized by grid-like subgraphs, like polycyclic aromatic hydrocarbons.

\textbf{SCAN} \cite{SCAN_clustering_Xu,SCAN_sparsification_Chen}, used for node clustering and graph sparsification, defines the structural similarity score for edges as:
\begin{equation}
SCAN(u, v) = \frac{|\mathcal{N}(u) \cap \mathcal{N}(v)| + 1}{\sqrt{(\mathrm{deg}(u) + 1)(\mathrm{deg}(v) + 1)}}
\end{equation}
SCAN scores were designed to detect the edges critical for graph connectivity, and those corresponding to outliers. In molecules, the distribution of SCAN scores can easily distinguish between linear structures (e.g. alkanes with long carbon chains), where the scores are low in general, and well-connected, ring-rich molecules (e.g. steroids).

Molecular graph classification relies heavily on atom and bond features, meaning that baselines utilizing only graph topology are not expressive enough. In fact, purely feature-based \q{molecular fingerprint} baseline from \cite{GNN_fair_comparison_Errica}, using only atom counts (i.e. counts of different chemical elements), can outperform some GNNs. Therefore, MOLTOP incorporates two such features: atomic numbers and bond types. They are the most apparent and consistently available features, universally employed by GNNs for analyzing molecules \cite{Pretraining_GNNs_Hu,OGB_Hu,AttentiveFP_Xiong}.

For atoms, we one-hot encode the atomic numbers up to 89 (with zero marking unknown types). We discard actinides and all further molecules, since they are all radioactive and extremely rarely used. For each chemical element, we compute its mean, standard deviation and sum (total count in the molecule) as graph-level features. We do the same for bonds, with $5$ possible types (single, double, triple, aromatic, or miscellaneous). In principle, one could add further features the same way, if they are known to be important for a given problem, e.g. chirality.

All features are computed for each graph independently, with the same number of bins $n_{bins}$ for all histograms. This results in $11 \cdot n_{bins} + 90 \cdot 3 + 5 \cdot 3$. This can result in many all-zero features, especially for atom features. We simply drop all such columns, based on the training data. Therefore, the number of features is often significantly reduced after this step.

The only hyperparameter of the feature extraction is the number of bins for histograms $n_{bins}$. In other works \cite{LDP,LTP,Weave_Kearnes} this is either a hyperparameter, requiring tuning, or just an arbitrarily set number. For MOLTOP, we propose a data-driven solution instead, setting the number of bins equal to the median size of the molecules (i.e. number of atoms) in the dataset. This is motivated by the fact that molecule sizes for drug-like compounds typically follow a right-skewed, single-modal distribution with number of atoms very rarely exceeding $50$ (see the supplementary material for plots). This fact is often used in medicinal chemistry, e.g. by Lipinski's rule of 5 \cite{Lipinski_rule_of_5}. Therefore, for the vast majority of data, it would bring little benefit to use a high number of bins. With this addition, MOLTOP does not require any hyperparameter tuning for feature extraction.

After feature extraction, we use Random Forest (RF) as a classifier. 
RF serves as an effective prediction model due to its low computational complexity and high scalability. Moreover, its performance is less sensitive to hyperparameter choices, unlike other commonly used classifiers such as SVMs or boosting methods \cite{Tunability_Probst}. It also natively supports multitask learning, which is common in molecular graph classification. The dimensionality of our representation is quite high, therefore we use larger number of trees and stronger regularization than default settings in Scikit-learn \cite{Scikit_learn_Pedregosa}, to better incorporate new features and prevent overfitting. Based on the average results on validation sets of MoleculeNet (detailed in \cref{experiments}), MOLTOP uses $1000$ trees, the entropy as splitting criterion, and minimum of $10$ samples to perform a split. Those reasonable defaults make it a hyperparameter-free method, making it extremely easy to use and computationally cheap, which is important for baseline methods.

\subsection{Complexity analysis}
\label{subsection_complexity_analysis}

The computational complexity of MOLTOP is the sum of complexities of its features, since they are computed independently. Vertex degree features have complexity $O(|E|)$ \cite{LDP}. Computing EBC has complexity $O(|V||E|)$ \cite{EBC_complexity_Ulrik}. Calculation of both ARI and SCAN Structural Similarity scores for all edges is pessimistically $O(|V||E|)$, but the expected complexity for molecular graphs is $O(|E|)$ due to their sparsity (for proofs, see the supplementary material). The total complexity of feature extraction is thus $O(|V||E|)$.

\section{Experiments and results}
\label{experiments}

For the main evaluation of the proposed method, we selected $8$ classification datasets from MoleculeNet benchmark \cite{MoleculeNet_Wu} (described in detail in the supplementary material), the most widely used molecular graph classification benchmark. For fair evaluation, we used deterministic scaffold split, with the splits provided by OGB \cite{OGB_Hu}. This setting is much more challenging than random split, which does not enforce out-of-distribution generalization. In addition, $5$ of those datasets are multitask, including massively multitask ToxCast dataset with $617$ targets. In all cases, we follow the recommendation from \cite{Pretraining_GNNs_Hu}, training $10$ models with different random seeds, and we report mean and standard deviation of AUROC.

For implementation of MOLTOP feature extraction, we used PyTorch Geometric \cite{PyTorch_Geometric_Fey} and NetworKit \cite{NetworKit_Angriman}. Those frameworks provide efficient data structures and parallel processing. For Random Forest, we use Scikit-learn \cite{Scikit_learn_Pedregosa}. Since this implementation does not allow missing labels in the training set, we fill them with zeros. This is acceptable, since those tasks are already imbalanced, and such adjustment makes this even a bit more challenging.

\subsection{Validation set experiments}
\label{section_validation_set_experiments}

During initial experiments, we verified our modelling choices by using average AUROC on validation sets. This setting was chosen, since we aimed to design a general baseline, that performs well on average for molecular classification. We started with only degree features with $50$ bins (inspired by \cite{LTP}), and added proposed improvements one by one. First, we validated that adding other topological descriptors, e.g. other centrality scores than EBC or other neighborhood overlap than ARI, gave results worse or similar to our proposed descriptors. Next, we confirmed that using all proposed statistics of atoms and bonds is crucial. Furthermore, we verified that using median molecule size as the number of bins gave results better or comparable to manual tuning. Lastly, we performed hyperparameter tuning of Random Forest, and resulting values were 1000 trees, entropy splitting criterion, and minimum of 10 samples to perform a split. Those align with our postulate to use more trees and stronger regularization.

We summarize the impact of adding described improvements in \cref{table_improvements_results} (for more detailed tables see the supplementary material). We report average AUROC and standard deviation for test sets of all $8$ MoleculeNet datasets. All proposed changes improve the results, in particular the introduction of atom and bonds features in addition to pure topology. This shows that effective baselines for molecular data have to use both structure and domain-relevant features.

\begin{table}[ht!]
\vskip 0.05in
\caption{The results of model improvements.}
\label{table_improvements_results}
\centering
\resizebox{\columnwidth}{!}{
\begin{tabular}{|c|c|}
\hline
Model                                   & Avg. AUROC $\uparrow$ \\ \hline
Degree features, 50 bins                            & 63.8$\pm$0.6          \\ \hline
Add topological edge features                & 65.5$\pm$0.9          \\ \hline
Add atoms and bonds features            & 69.0$\pm$0.8          \\ \hline
Median bins                             & 69.4$\pm$0.9          \\ \hline
Reduce degrees bins, drop constant features & 70.4$\pm$0.7          \\ \hline
Use tuned Random Forest hyperparameters      & 72.5$\pm$0.5          \\ \hline
\end{tabular}
}
\vskip -0.1in
\end{table}

\subsection{MoleculeNet classification}
\label{subsection_moleculenet_classification}

\begin{table*}[ht!]
\caption{Classification results on MoleculeNet. \q{Pretr} denotes if the model is pretrained. The best result for each dataset is bolded, and also for each model group (general-purpose GNNs, molecular GNNs, baselines), the best model and its average AUROC and rank are bolded.}
\label{table_main_results}
\resizebox{\textwidth}{!}{
\begin{tabular}{|c|c|c|c|c|c|c|c|c|c|c|c|}
\hline
\textbf{Model}        & \textbf{Pretr.} & \textbf{BACE}           & \textbf{BBBP}           & \textbf{HIV}            & \textbf{ClinTox}        & \textbf{MUV}            & \textbf{SIDER}          & \textbf{Tox21} & \textbf{ToxCast}        & \textbf{\begin{tabular}[c]{@{}c@{}}Avg.\\ AUROC $\uparrow$\end{tabular}} & \textbf{\begin{tabular}[c]{@{}c@{}}Avg.\\ rank $\downarrow$\end{tabular}} \\ \hline
GIN                   & No              & 70.1 $\pm$ 5.4          & 65.8 $\pm$ 4.5          & 75.3 $\pm$ 1.9          & 71.8 $\pm$ 2.5          & 58.0 $\pm$ 4.4          & 57.3 $\pm$ 1.6          & 74.0 $\pm$ 0.8 & 63.4 $\pm$ 0.6          & 67 $\pm$ 2.7                                                             & 14.4                                                                      \\ \hline
\textbf{GIN}          & Yes             & 84.5 $\pm$ 0.7          & 68.7 $\pm$ 1.3          & 79.9 $\pm$ 0.7          & 81.3 $\pm$ 2.1          & 72.6 $\pm$ 1.5          & 62.7 $\pm$ 0.8          & 78.1 $\pm$ 0.6 & 65.7 $\pm$ 0.6          & \textbf{74.2 $\pm$ 1}                                                    & \textbf{4.3}                                                              \\ \hline
GCN                   & No              & 73.6 $\pm$ 3.0          & 64.9 $\pm$ 3.0          & 75.7 $\pm$ 1.1          & 73.2 $\pm$ 1.4          & 65.8 $\pm$ 4.5          & 60.0 $\pm$ 1.0          & 74.9 $\pm$ 0.8 & 63.3 $\pm$ 0.9          & 68.9 $\pm$ 2                                                             & 12.3                                                                      \\ \hline
GCN                   & Yes             & 82.3 $\pm$ 3.4          & 70.6 $\pm$ 1.6          & 78.2 $\pm$ 0.6          & 79.4 $\pm$ 1.8          & 63.6 $\pm$ 1.7          & 62.4 $\pm$ 0.5          & 75.8 $\pm$ 0.3 & 65.3 $\pm$ 0.1          & 72.2 $\pm$ 1.3                                                           & 6.1                                                                       \\ \hline
GraphSAGE             & No              & 72.5 $\pm$ 1.9          & 69.6 $\pm$ 1.9          & 74.4 $\pm$ 0.7          & 72.7 $\pm$ 1.4          & 59.2 $\pm$ 4.4          & 60.4 $\pm$ 1.0          & 74.7 $\pm$ 0.7 & 63.3 $\pm$ 0.5          & 68.4 $\pm$ 1.6                                                           & 12.1                                                                      \\ \hline
GraphSAGE             & Yes             & 80.7 $\pm$ 0.9          & 63.9 $\pm$ 2.1          & 76.2 $\pm$ 1.1          & 78.4 $\pm$ 2.0          & 60.7 $\pm$ 2.0          & 60.7 $\pm$ 0.5          & 76.8 $\pm$ 0.3 & 64.9 $\pm$ 0.2          & 70.3 $\pm$ 1.1                                                           & 9.6                                                                       \\ \hline
GraphCL               & Yes             & 68.7 $\pm$ 7.8          & 67.5 $\pm$ 3.3          & 75.0 $\pm$ 0.4          & 78.9 $\pm$ 4.2          & 77.1 $\pm$ 1.0          & 60.1 $\pm$ 1.3          & 75.0 $\pm$ 0.3 & 62.8 $\pm$ 0.2          & 70.6 $\pm$ 2.3                                                           & 11.1                                                                      \\ \hline
JOAO                  & Yes             & 72.9 $\pm$ 2.0          & 66.0 $\pm$ 0.6          & 76.6 $\pm$ 0.5          & 66.3 $\pm$ 3.9          & 77.0 $\pm$ 2.2          & 60.7 $\pm$ 1.0          & 74.4 $\pm$ 0.7 & 62.7 $\pm$ 0.6          & 69.6 $\pm$ 1.4                                                           & 11.4                                                                      \\ \hline
\hline
D-MPNN                & No              & 80.9 $\pm$ 0.6          & 71.0 $\pm$ 0.3          & 77.1 $\pm$ 0.5          & \textbf{90.6 $\pm$ 0.6} & 78.6 $\pm$ 1.4          & 57.0 $\pm$ 0.7          & 75.9 $\pm$ 0.7 & 65.5 $\pm$ 0.3          & 74.6 $\pm$ 0.6                                                           & 5.9                                                                       \\ \hline
AttentiveFP           & No              & 78.4 $\pm$ 2.2          & 64.3 $\pm$ 1.8          & 75.7 $\pm$ 1.4          & 84.7 $\pm$ 0.3          & 76.6 $\pm$ 1.5          & 60.6 $\pm$ 3.2          & 76.1 $\pm$ 0.5 & 63.7 $\pm$ 0.2          & 72.5 $\pm$ 1.4                                                           & 9                                                                         \\ \hline
GROVER                & Yes             & 81.0 $\pm$ 1.4          & 69.5 $\pm$ 0.1          & 68.2 $\pm$ 1.1          & 76.2 $\pm$ 3.7          & 67.3 $\pm$ 1.8          & 65.4 $\pm$ 0.1          & 73.5 $\pm$ 0.1 & 65.3 $\pm$ 0.5          & 70.8 $\pm$ 1.1                                                           & 9.1                                                                       \\ \hline
GraphMVP              & Yes             & 76.8 $\pm$ 1.1          & 68.5 $\pm$ 0.2          & 74.8 $\pm$ 1.4          & 79.0 $\pm$ 2.5          & 75.0 $\pm$ 1.4          & 62.3 $\pm$ 1.6          & 74.5 $\pm$ 0.4 & 62.7 $\pm$ 0.1          & 71.7 $\pm$ 1.1                                                           & 10.3                                                                      \\ \hline
GraphMVP-C            & Yes             & 81.2 $\pm$ 0.9          & 72.4 $\pm$ 1.6          & 77.0 $\pm$ 1.2          & 77.5 $\pm$ 4.2          & 75.0 $\pm$ 1.0          & 63.9 $\pm$ 1.2          & 74.4 $\pm$ 0.2 & 63.1 $\pm$ 0.4          & 73.1 $\pm$ 1.3                                                           & 7.1                                                                       \\ \hline
\textbf{GEM}          & Yes             & \textbf{85.6 $\pm$ 1.1} & \textbf{72.4 $\pm$ 0.4} & 80.6 $\pm$ 0.9          & 90.1 $\pm$ 1.3          & \textbf{81.7 $\pm$ 0.5} & \textbf{67.2 $\pm$ 0.4} & \textbf{78.1 $\pm$ 0.1} & \textbf{69.2 $\pm$ 0.4} & \textbf{78.1 $\pm$ 0.6}                                                  & \textbf{1.3}                                                              \\ \hline
\hline
ECFP               & No              & 83.8 $\pm$ 0.4          & 68.6 $\pm$ 0.5          & 76.3 $\pm$ 0.6          & 71.7 $\pm$ 1.6          & 66.9 $\pm$ 1.3          & 67.1 $\pm$ 0.4          & 72.8 $\pm$ 0.2          & 60.4 $\pm$ 0.4          & 71 $\pm$ 0.7                                                             & 9.9                                                                       \\ \hline
\q{molecular fingerprint}    & No              & 71.5 $\pm$ 0.2          & 68.3 $\pm$ 0.3          & 65.5 $\pm$ 0.6          & 65.5 $\pm$ 0.9          & 49.8 $\pm$ 0.0          & 59.0 $\pm$ 0.2          & 63.5 $\pm$ 0.2          & 57.5 $\pm$ 0.1          & 62.6 $\pm$ 0.3                                                           & 17.3                                                                      \\ \hline
LDP                 & No              & 80.5 $\pm$ 0.3          & 63.3 $\pm$ 0.4          & 72.1 $\pm$ 0.4          & 58.2 $\pm$ 2.1          & 50.0 $\pm$ 0.9          & 59.8 $\pm$ 0.5          & 66.7 $\pm$ 0.2          & 59.5 $\pm$ 0.4          & 63.8 $\pm$ 0.3                                                           & 16                                                                        \\ \hline
LTP                 & No              & 80.7 $\pm$ 0.3          & 65.6 $\pm$ 0.3          & 73.0 $\pm$ 0.7          & 61.7 $\pm$ 1.7          & 53.2 $\pm$ 1.7          & 60.8 $\pm$ 0.5          & 67.7 $\pm$ 0.5          & 60.0 $\pm$ 0.4          & 65.3 $\pm$ 0.5                                                           & 13.8                                                                      \\ \hline

\textbf{MOLTOP}       & No              & 82.9 $\pm$ 0.2          & 68.9 $\pm$ 0.2          & \textbf{80.8 $\pm$ 0.3} & 73.6 $\pm$ 0.7          & 66.7 $\pm$ 1.9          & 66.0 $\pm$ 0.5          & 76.3 $\pm$ 0.2 & 64.4 $\pm$ 0.3          & \textbf{72.5 $\pm$ 0.5}                                                  & \textbf{6}                                                                \\ \hline
\end{tabular}
}
\end{table*}

We compared MOLTOP to $18$ other graph classification methods on MoleculeNet benchmark, with results in \cref{table_main_results}. We compare it to methods from three groups: general-purpose GNNs, GNNs designed specifically for molecular data, and graph classification baselines. This way, we verify not only that MOLTOP improves upon previous baselines, but also achieves strong performance in comparison to sophisticated, domain-specific models.

We include $8$ general-purpose GNNs: GIN, GCN and GraphSAGE from \cite{Pretraining_GNNs_Hu}, both with and without context prediction pretraining, as well as recent models based on contrastive learning, GraphCL \cite{GraphCL} and JOAO \cite{JOAO}. For GNNs designed specifically for molecular property prediction, we include multiple recent models utilizing different approaches to incorporating molecular features: D-MPNN \cite{DMPNN_Yang}, AttentiveFP \cite{AttentiveFP_Xiong}, GROVER \cite{GROVER_Rong} (large variant), GraphMVP \cite{GraphMVP} (regular and contrastive variants), and GEM \cite{GEM_Fang}. We also compare to four other baselines: purely topological LDP \cite{LDP} and LTP \cite{LTP}, purely feature-based \q{molecular fingerprint} from \cite{GNN_fair_comparison_Errica} (which uses atom counts as features), and ECFP, the molecular fingerprint commonly used as a baseline for GNNs \cite{DMPNN_Yang} (using default settings). For those baselines, we use Random Forest with 500 trees as a classifier, which follows \cite{LTP} and is a common setting in chemoinformatics.

Following best practices for statistical comparison of classifiers from \cite{Statistical_comparisons_Demsar}, we report average model rank across datasets, in addition to average AUROC. This metric is less influenced by outliers among scores, and therefore better measures how the model really performs on average. In particular, the ClinTox dataset often gives very unstable results \cite{GNN_pretraining_benchmark_Sun}, and the average rank should be less susceptible to this problem.

The main observation is that MOLTOP, under this fair comparison protocol, outperforms the majority of models on average, often by a large margin. In terms of average rank, it exceeds all GNNs without pretraining except for D-MPNN, which has almost identical average rank. It also has results better than most pretrained GNNs, even including recent, complex models like JOAO, GROVER and GraphMVP. This is particularly significant, since MOLTOP does not utilize any external knowledge like those models, nor did it require very costly pretraining on massive datasets. Our results are also notably stable, with low standard deviations, indicating the robustness of this approach.

MOLTOP does not require any pretraining, and requires only around $50$ minutes for the entire benchmark (with massively multitask ToxCast taking the majority of the time). In addition, it has very low standard deviations, indicating stable and robust behavior. This shows that fair comparison, using strong baselines, remains important even in the era of large pretrained models.

Outperforming GNNs can be explained by the global nature of features used by MOLTOP. Those models, while sophisticated, still rely on an inherently local message-passing paradigm, and especially without pretraining it is hard for them to fully understand molecular relations on limited data.

The only models that have better average rank than MOLTOP are pretrained GIN, D-MPNN, and GEM. However, using Wilcoxon signed-rank test (recommended by \cite{Statistical_comparisons_Demsar}) with $\alpha=0.05$, we determined that difference with GIN and D-MPNN performance is not statistically significant (p-values $0.547$ and $0.742$). Only GEM outperforms MOLTOP significantly (p-value $0.016$), but we note that it has an enormous computational cost, including fine-tuning and even inference, since it requires generation of multiple conformers and Merck molecular force field (MMFF) optimization. Those operations can easily take minutes per molecule, can often fail for molecules with complicated geometries (e.g. highly rotatable bonds), and are simply impossible in many cases, e.g. for compounds with disconnected components like salts. Therefore, MOLTOP always achieves results better or as good as GNNs, except for GEM, which has major practical downsides.

MOLTOP also improves upon other baselines by a large margin. Previous approaches like LDP, LTP and \q{molecular fingerprint} of \cite{GNN_fair_comparison_Errica} often fail to beat almost any GNNs, and thus are unsuitable for molecular data. Notably, we even outperform ECFP4 fingerprint, often used to compare again GNNs. This shows that improving upon existing baselines remains important for fair comparison.

We present additional comparisons with graph kernels in the supplementary material. We omit them here, because due to OOM errors they couldn't be computed on HIV and MUV datasets, meaning that we can't directly compare their average AUROC and rank to other models in \cref{table_main_results}.

\subsection{HIV leaderboard results}

We further evaluate MOLTOP on the HIV dataset featured in OGB leaderboard \cite{OGB_Hu}, comparing it to various cutting-edge models that do not provide results on the whole MoleculeNet benchmark. The results are shown in \cref{table_ogb_results}. While this is the same HIV dataset as used before, here we are not allowed to use the validation data (due to leaderboard rules), even when no hyperparameters are tuned.

\begin{table}[b]
\vspace{0.2cm}
\caption{Selected results on HIV leaderboard in OGB. MOLTOP results are marked in bold.}
\resizebox{\columnwidth}{!}{
\begin{tabular}{|c|c|c|c|}
\hline
Method          & Rank $\downarrow$ & Test AUROC $\uparrow$     & Valid AUROC $\uparrow$    \\ \hline
CIN             & 9                 & 80.94$\pm$0.57          & 82.77$\pm$0.99          \\ \hline
GSAT            & 10                 & 80.67$\pm$0.95          & 83.47$\pm$0.31          \\ \hline
Graphormer      & 14                & 80.51$\pm$0.53          & 83.10$\pm$0.89          \\ \hline
\textbf{MOLTOP} & \textbf{13}       & \textbf{80.42$\pm$0.25} & \textbf{80.33$\pm$0.54} \\ \hline
P-WL            & 15                & 80.39$\pm$0.40          & 82.79$\pm$0.59          \\ \hline
Directional GSN & 15                & 80.39$\pm$0.90          & 84.73$\pm$0.96          \\ \hline
PNA             & 18                & 79.05$\pm$1.32          & 85.19$\pm$0.99          \\ \hline
DeeperGCN       & 21                & 78.58$\pm$1.17          & 84.27$\pm$0.63          \\ \hline
\end{tabular}
}
\label{table_ogb_results}
\end{table}

Since there are currently $34$ models on the leaderboard, here we present a few selected ones. MOLTOP achieves $14$-th rank, outperforming well-known PNA \cite{PNA_Corso} and DeeperGCN \cite{DeeperGCN_Guohao}, and coming very close to Graphormer \cite{Graphormer_Ying}, GSAT \cite{GSAT_Miao} and CIN \cite{CIN_Bodnar}. It is also narrowly better than very powerful Directional GSN \cite{GSN_Bouritsas}. If we lift the limitation of not using the validation set, which is quite artificial for hyperparameter-free MOLTOP, it gets $80.8\%$ AUROC and outperforms both GSAT and Graphormer.

\subsection{Peptides classification}

In order to further evaluate the out-of-distribution generalization abilities of MOLTOP, we utilize the peptides-func dataset from LRGB benchmark \cite{LRGB_peptides_func}, concerning peptide function classification. The characteristics of this data are very different from MoleculeNet, with peptides being much bigger molecules, with larger diameter and long-range dependencies. We do not perform any tuning, requiring the hyperparameter-free baseline to perform reasonably well even on this very different domain.

In \cref{table_peptides_func_results}, we compare to the results from \cite{LRGB_peptides_func}. Remarkably, MOLTOP outperforms all GNNs, including graph transformers specifically designed for this task, e.g. SAN with RWSE embeddings. At the same time, it is much more stable, with very low standard deviation. This shows that it indeed works very well as a baseline, even for novel datasets and molecular domains.

\begin{table}[]
\vspace{0.2cm}
\caption{Results on peptides-func dataset. Best result is bolded.}
\centering
\begin{tabular}{|c|c|}
\hline
Method            & Test average precision $\uparrow$    \\ \hline
GCN               & 59.30 $\pm$ 0.23          \\ \hline
GCNII             & 55.43 $\pm$ 0.78          \\ \hline
GINE              & 54.98 $\pm$ 0.79          \\ \hline
GatedGCN          & 58.64 $\pm$ 0.77          \\ \hline
GatedGCN+RWSE     & 60.69 $\pm$ 0.35          \\ \hline
Transformer+LapPE & 63.26 $\pm$ 1.26          \\ \hline
SAN+LapPE         & 63.84 $\pm$ 1.21          \\ \hline
SAN+RWSE          & 64.39 $\pm$ 0.75          \\ \hline
MOLTOP            & \textbf{64.59 $\pm$ 0.05} \\ \hline
\end{tabular}
\label{table_peptides_func_results}
\vspace{-0.5cm}
\end{table}

\subsection{Time efficiency benchmark}

While MOLTOP has very low feature extraction complexity, as outlined in \cref{subsection_complexity_analysis}, we also measure wall time for both feature extraction and RF training, summarized in \cref{table_times}. On four smallest datasets, it requires less than ten seconds for both, and at most about a minute for a further three datasets. In particular, feature extraction takes only $35$ seconds on over $15$ thousands of peptides, which are very large molecules by the standard of molecular graph classification, which mostly concerns small, drug-like compounds. On MUV, which is by far the largest in terms of the number of molecules, feature extraction still takes only about $2.5$ minutes. In general, we note that MOLTOP feature extraction is embarrassingly parallel, and can process almost arbitrary number of molecules, given enough CPUs.

\begin{table}[]
\vspace{0.2cm}
\caption{MOLTOP timings.}
\centering
\resizebox{\columnwidth}{!}{
\begin{tabular}{|c|c|c|c|c|}
\hline
Dataset       & \# molecules & \# tasks & \begin{tabular}[c]{@{}c@{}}Feature\\ extraction {[}s{]}\end{tabular} & Training {[}s{]} \\ \hline
BACE          & 1513         & 1        & 2                                                                    & 1                \\ \hline
BBBP          & 2039         & 1        & 3                                                                    & 1                \\ \hline
HIV           & 41127        & 1        & 57                                                                   & 6                \\ \hline
ClinTox       & 1478         & 17       & 2                                                                    & 1                \\ \hline
MUV           & 93087        & 2        & 144                                                                  & 38               \\ \hline
SIDER         & 1427         & 27       & 2                                                                    & 2                \\ \hline
Tox21         & 7831         & 12       & 11                                                                   & 4                \\ \hline
ToxCast       & 8575         & 617      & 12                                                                   & 232              \\ \hline
Peptides-func & 15535        & 10       & 35                                                                   & 5                \\ \hline
\end{tabular}
}
\label{table_times}
\end{table}

Overall, the ToxCast takes the most time, but for the training part, due to being massively multitask. This is most likely the artifact of the implementation, since such dataset are rare and Scikit-learn implementation of RF is not particularly optimized for those cases. In fact, this is the only dataset in molecular property prediction that we are aware of with such huge number of tasks.

For comparison with GNNs, we focus on peptides-func dataset, for which \cite{LRGB_peptides_func} provides wall times. Computing LapPE or RWSE embeddings alone, which are necessary for feature augmentation to get reasonable performance of GNNs on this dataset (due to long-range dependencies), takes about a minute. This does not even take into consideration the training time of GNN model itself.

Finally, since MOLTOP is hyperparameter-free, it does not require time for tuning for new datasets. This is especially advantageous in comparison to GNNs, which require extensive tuning of at least learning rate and regularization parameters for new datasets. This increases their training cost multiple times, while MOLTOP can just be used as-is.

\subsection{Feature importance analysis}
\label{section_feature_importance_analysis}

We additionally validate the importance of features leveraging Random Forest average decrease in entropy. This metric is effective in identifying the features that are most useful for the model. The importance of a feature is the sum of importances of its histograms bins, since each bin is treated as a separate feature for the classifier. Next, we average values obtained from $10$ classifiers on each dataset, based on different random seeds. To aggregate this information for the entire benchmark, we further average the importances for all $8$ datasets. This is shown in \cref{plot_averave_feature_importances}.

\begin{figure}[]
\centering
\includegraphics[width=\columnwidth]{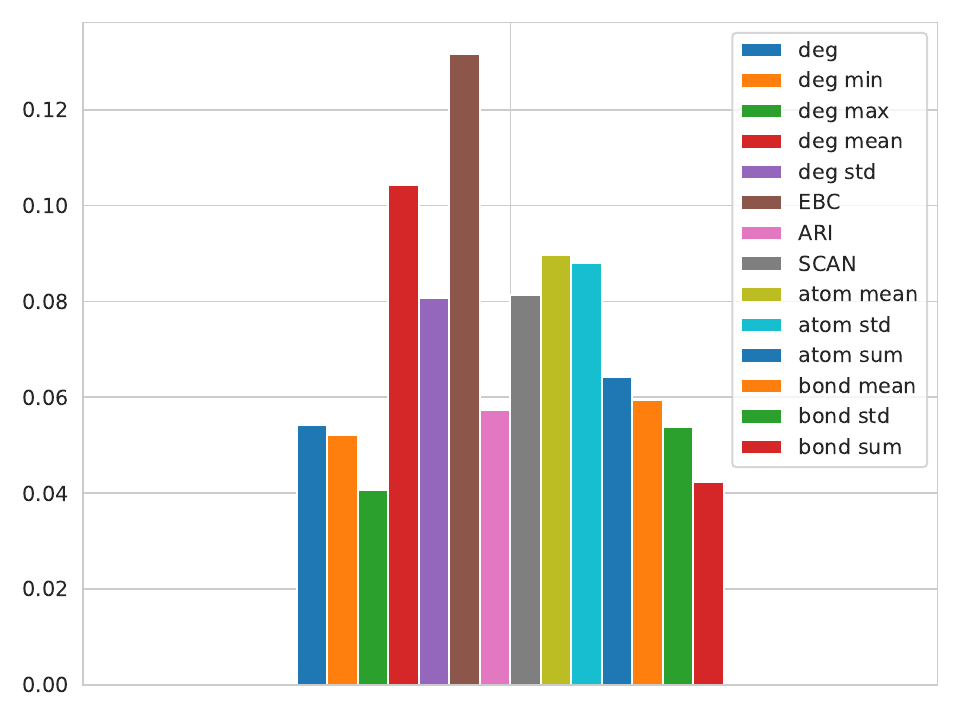}
\caption{Average MOLTOP feature importances.}
\label{plot_averave_feature_importances}
\vskip 0.2in
\end{figure}

The main outcome is that all features are useful, as there are none with very low importance. The most influential feature is EBC, which highlights that global information is particularly important for molecular data, and validates our initial belief. The least useful feature is the maximal degree of neighbors, which is expected, as node degrees are typically very low in chemistry. For additional ablation studies, see the supplementary material.

\subsection{Expressivity experiments}
\label{section_expressivity_experiments}

Lastly, we analyzed the expressive power of MOLTOP topological features in distinguishing graphs. It is typically represented using a hierarchy of $k$-dimensional Weisfeiler-Lehman isomorphism tests \cite{Weisfeiler_Lehman,GNN_expressiveness_Sato}, but can also be verified by using particular classes of graphs, which are known to be hard to distinguish for computational methods. Typical GNNs are at most as powerful as 1-WL test \cite{GIN}, but with specific extensions, often utilizing topological descriptors in forms of shortest paths or subgraphs counting, GNNs become more powerful \cite{Graphormer_Ying,GSN_Bouritsas}.

We verified the discriminative power of MOLTOP feature vectors using \textit{graph8c} and \textit{sr25} datasets \cite{GNN_expressivity_testing_Balcilar}. Our topological features are all integers after histogram aggregation, so we deem two graphs to be different if they have different values of any feature. We perform paired comparisons of graphs this way, where the number of pairs is $61$M for \textit{graph8c} and $105$ for \textit{sr25}. We report number of errors, i.e. undistinguished pairs, in \cref{table_graph_distinguishing_errors}.

MOLTOP achieves very good results, showing high power in distinguishing graphs. It outperforms all message-passing GNNs, probably due to usage of features that incorporate more global information. Additionally, it performs almost as well as PPGN \cite{PPGN_Maron} and \mbox{GNNML3} \cite{GNN_expressivity_testing_Balcilar}, which are provably as powerful as $2$-FWL test, equivalent to $3$-WL test. It also achieves perfect result on \textit{sr25}, which consists of strongly regular graphs. This is particularly exceptional, as they are $3$-WL equivalent \cite{3_WL_properties_Arvind}, which means that MOLTOP can distinguish graphs for which even $3$-WL test fails.

We provide examples of graphs distinguishable by MOLTOP, but not e.g. by 1-WL test, in the supplementary material.

\begin{table}[!htb]
\caption{The number of undistinguished pairs of graphs in \textit{graph8c} and \textit{sr25}.}
\centering
\begin{tabular}{|c|c|c|}
\hline
Model   & graph8c $\downarrow$ & sr25 $\downarrow$ \\ \hline
MLP     & 293K                 & 105               \\ \hline
GCN     & 4775                 & 105               \\ \hline
GAT     & 1828                 & 105               \\ \hline
GIN     & 386                  & 105               \\ \hline
ChebNet & 44                   & 105               \\ \hline
PPGN    & \textbf{0}           & 105               \\ \hline
GNNML1  & 333                  & 105               \\ \hline
GNNML3  & \textbf{0}           & 105               \\ \hline
MOLTOP  & 3                    & \textbf{0}        \\ \hline
\end{tabular}
\label{table_graph_distinguishing_errors}
\vspace{-0.5cm}
\end{table}

\section{Conclusion}
\label{conclusion}

We presented a new type of molecular graph embedding, which leverages local and global structural information aggregated from vertex and edge descriptors, as well as basic semantics of bonds and atoms. Combined with low parameter classification using Random Forests, it forms a robust baseline algorithm for molecular property prediction called MOLTOP. The key advantages of MOLTOP are: low computational cost, no hyperparameter tuning required, and high discriminative power, which surpasses 1-WL isomorphism test. Based on fair evaluation protocols and deterministic scaffold splits, we show that MOLTOP is surprisingly competitive with GNNs, including out-of-generalization applications to new datasets. With additional verification of results using Wilcoxon signed-rank test, we show that our proposed model is better or as good as all baselines and GNNs, except for GEM model, which uses computationally expensive and error-prone 3D molecular modelling.

In the future work, we plan to experiment with incorporating additional features, and adapt this approach to e.g. materials chemistry. We also want to more thoroughly analyze the theoretical aspects of feature descriptors and their discriminative abilities in terms of WL hierarchy.

We conclude that strong baselines, such as MOLTOP, are still important to gain deep insights into advances of GNN pre-training and assessing benefits of incorporating spatial or structural information, especially in the experimental setups with limited computational budgets.



\begin{ack}

Research was supported by the funds assigned by Polish Ministry of Science and Higher Education to AGH University of Krakow, and by the grant from \q{Excellence Initiative - Research University} (IDUB) for the AGH University of Krakow. We gratefully acknowledge Poland's high-performance Infrastructure PLGrid ACK Cyfronet AGH for providing computer facilities and support within computational grant. We would like to thank Alexandra Elbakyan for her work and support for accessibility of science.

\end{ack}



\bibliography{bibliography}

\clearpage
\appendix
\onecolumn


\begin{center}
  \huge \textbf{Supplementary information} \par
\end{center}

\section{Descriptors histograms examples}

Here, we visualize the discriminative power of proposed topological descriptors on two example molecules (\cref{figure_dpcc_paclitaxel}): Dipalmitoylphosphatidylcholine (DPCC), a phospholipid used as a pulmonary surfactant, and Paclitaxel, used in cancer treatment. Histograms of EBC, ARI and SCAN (with $5$ bins and normalized for readability) are presented in \cref{hist_EBC_ARI_SCAN}. It should be noted that those distributions follow chemical intuitions outlined in the Methods section in the main paper.

\begin{figure}[h!]
  \centering
\subfloat{\includegraphics[width=0.45\columnwidth]{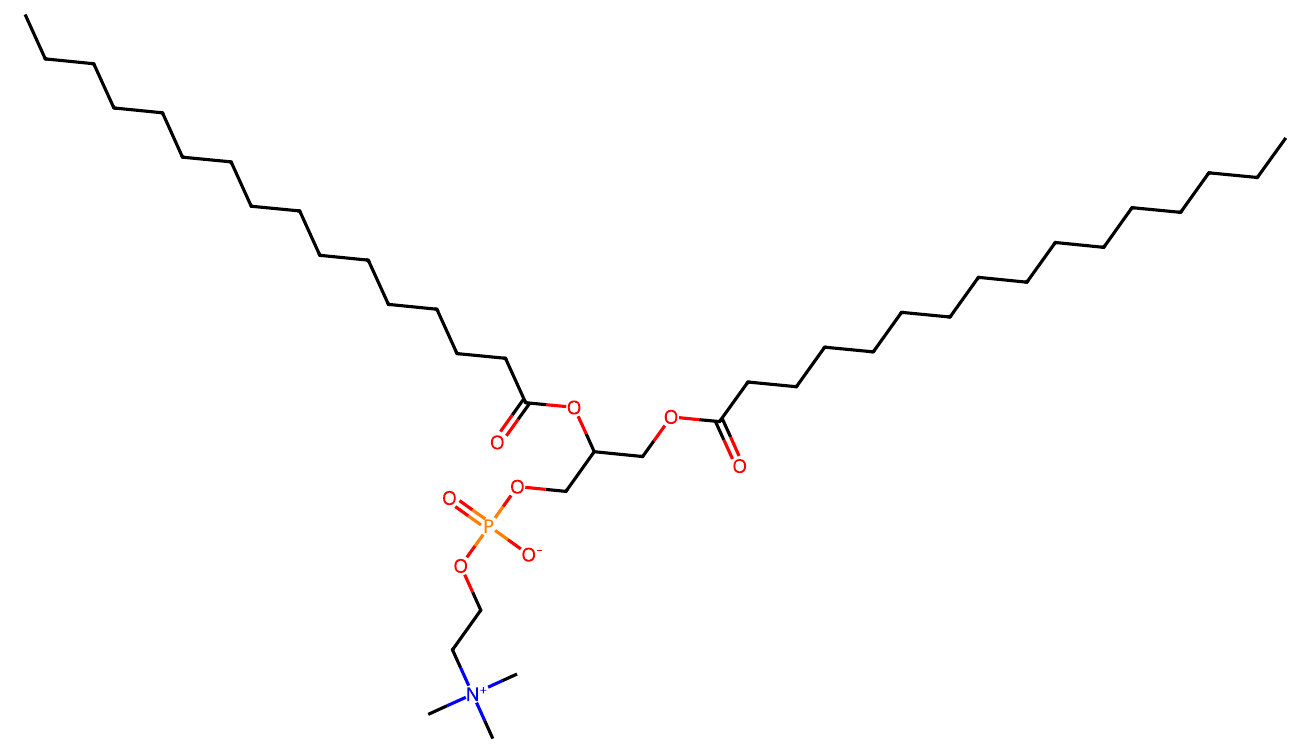}}
\subfloat{\includegraphics[width=0.45\columnwidth]{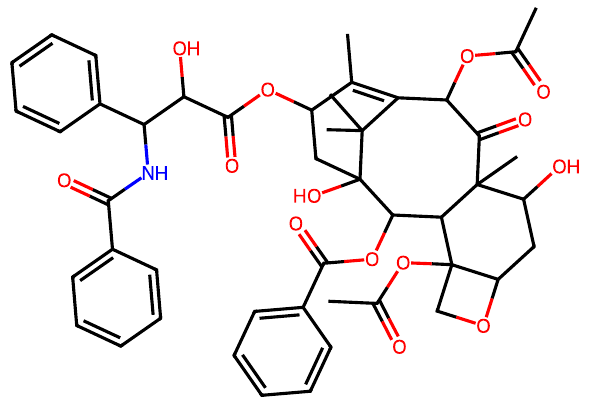}}
\caption{DPCC and Paclitaxel molecules.}
\label{figure_dpcc_paclitaxel}
\end{figure}

\vskip 1in

\begin{figure}[h!]
  \centering
\subfloat{\includegraphics[width=0.33\columnwidth]{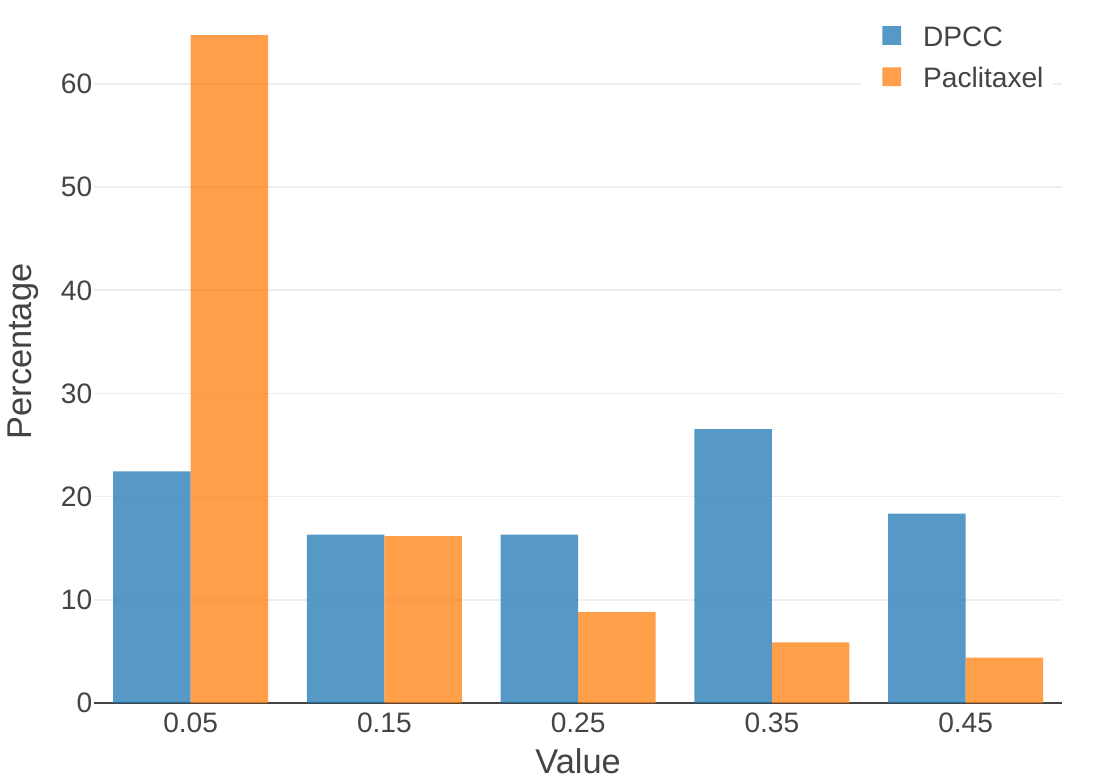}}
\subfloat{\includegraphics[width=0.33\columnwidth]{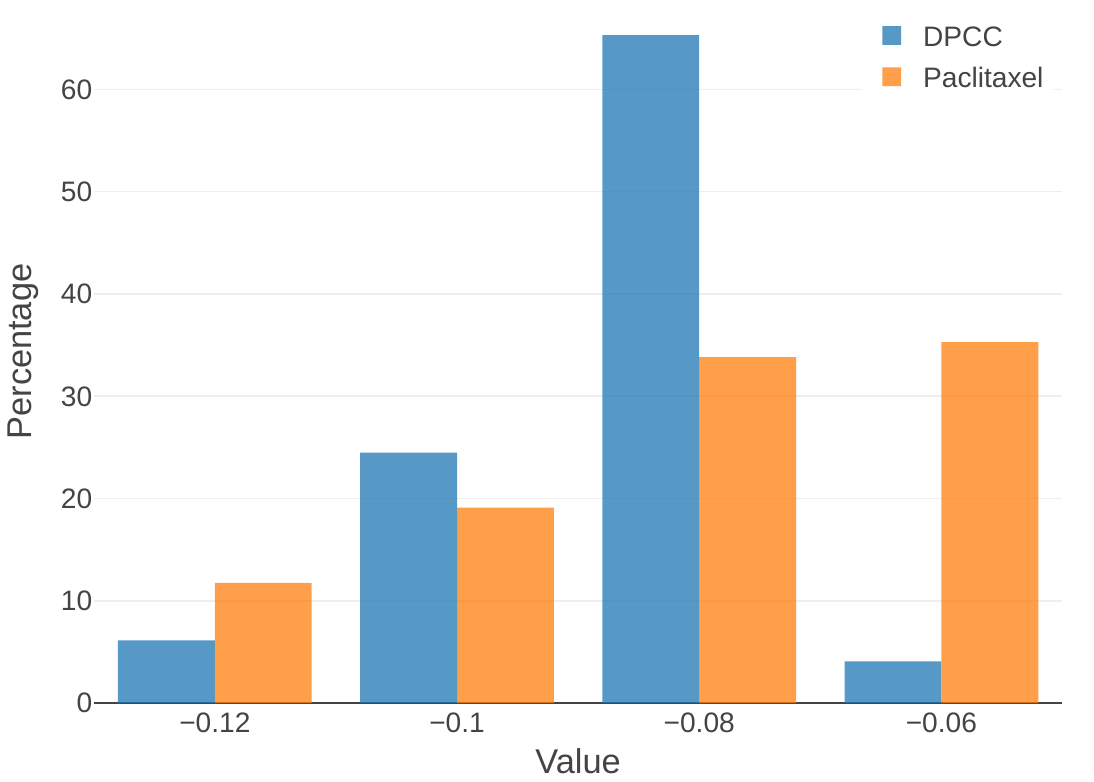}}
\subfloat{\includegraphics[width=0.33\columnwidth]{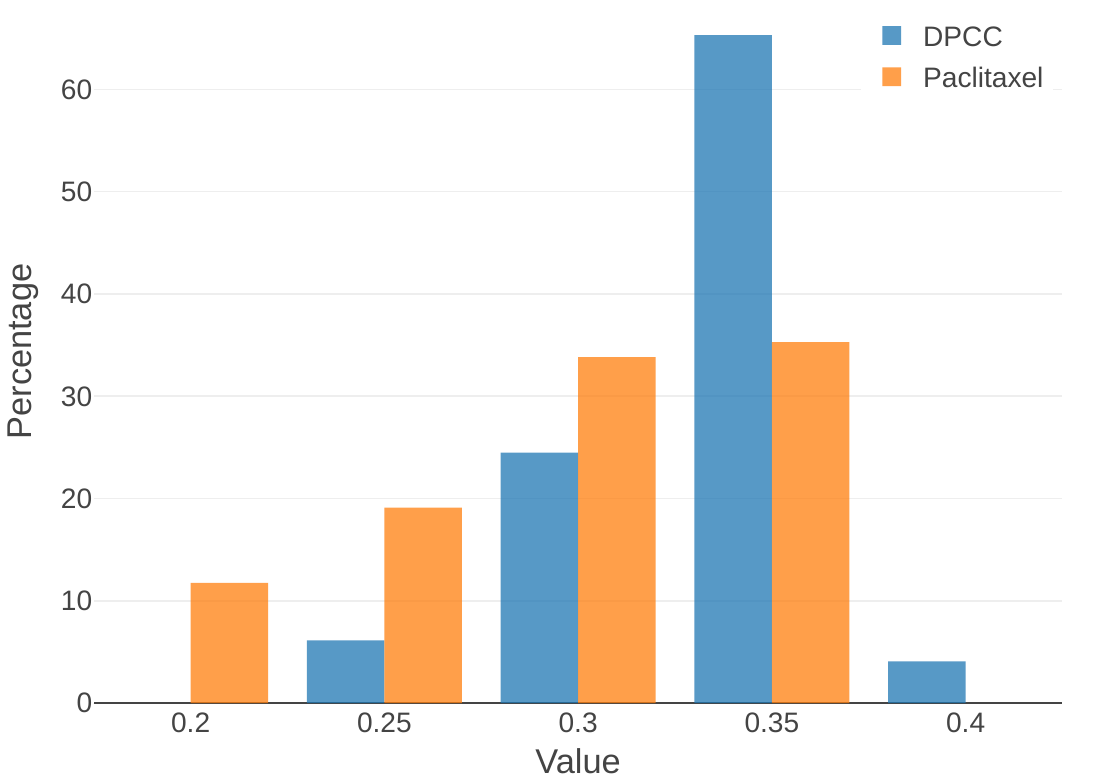}}
    \caption{Normalized histograms of edge descriptors for DPCC and Paclitaxel: (a) EBC (b) ARI (c) SCAN.}
    \label{hist_EBC_ARI_SCAN}
\end{figure}

\clearpage
\section{Datasets descriptions}
\label{appendix_datasets_descriptions}

Here, we present a short description of datasets from MoleculeNet, as well as peptides-func from LRGB \cite{LRGB_peptides_func}, including their basic properties. The statistics are summarized in \cref{table_appendix_datasets_statistics}, and we describe all datasets below. Additionally, in \cref{figure_molevules_hiv_toxcast} we present distributions of molecules sizes for HIV and ToxCast datasets. The distributions for the rest of the datasets are very similar, so we omit them for brevity.
\begin{itemize}
    \item \textbf{BACE} \cite{BACE} - binary prediction of binding results for a set of inhibitors of human $\beta$-secretase 1 (BACE-1).
    \item \textbf{BBBP} \cite{BBBP} - prediction whether a compound is able to penetrate the blood-brain barrier.
    \item \textbf{HIV} \cite{HIV} - prediction whether the molecule can inhibit the HIV replication.
    \item \textbf{ClinTox} \cite{ClinTox} - database of drugs approved and rejected by FDA for toxicity reasons. Two tasks concern prediction of drug toxicity during clinical trials and during FDA approval process.
    \item \textbf{MUV} \cite{MUV} - the Maximum Unbiased Validation (MUV) has been designed for validation of virtual screening techniques, consisting of $17$ tasks based on PubChem BioAssay combined with a refined nearest neighbor analysis.
    \item \textbf{SIDER} \cite{SIDER} - the Side Effect Resource database, considering prediction of adverse side effects of drugs on $27$ system organ classes.
    \item \textbf{Tox21} \cite{Tox21} - coming from $2014$ Tox21 Data Challenge, this dataset concerns prediction of 12 toxicity targets.
    \item \textbf{ToxCast} \cite{ToxCast} - toxicology measurements for $617$ targets from a large scale in vitro high-throughput screening.
    \item \textbf{peptides-func} \cite{LRGB_peptides_func} - functions of peptides (small proteins), based on SATPdb data.
\end{itemize}

\begin{table}[h!]
\caption{MoleculeNet datasets statistics.}
\centering
\begin{tabular}{|c|c|c|c|}
\hline
Dataset & \# molecules & \# tasks & Median molecule size \\ \hline
BACE    & 1513         & 1        & 33                   \\ \hline
BBBP    & 2039         & 1        & 23                   \\ \hline
HIV     & 41127        & 1        & 23                   \\ \hline
ClinTox & 1478         & 2        & 23                   \\ \hline
MUV     & 93087        & 17       & 24                   \\ \hline
SIDER   & 1427         & 27       & 25                   \\ \hline
Tox21   & 7831         & 12       & 16                   \\ \hline
ToxCast & 8575         & 617      & 16                   \\ \hline
\end{tabular}
\label{table_appendix_datasets_statistics}
\end{table}

\begin{figure}[ht!]
\centering
\begin{tabular}{cc}
\includegraphics[width=0.40\columnwidth]{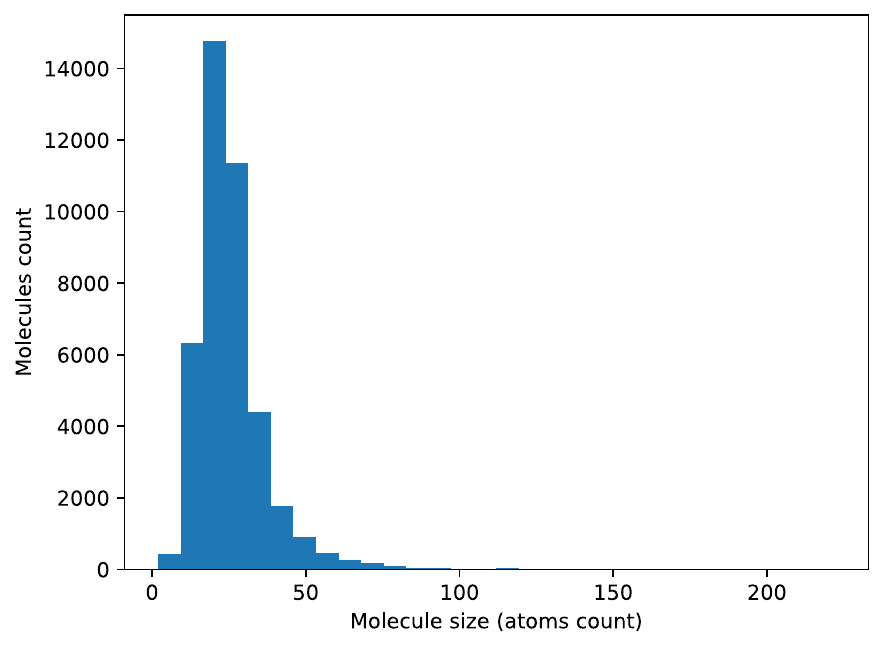}
&
\includegraphics[width=0.40\columnwidth]{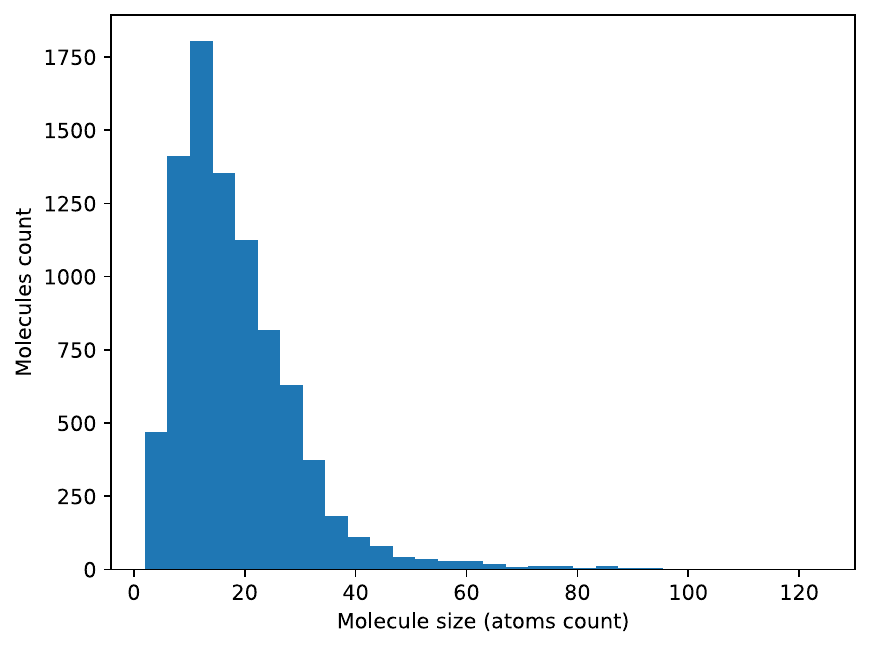} \\
(a) & (b)
\end{tabular}
\caption{Molecule sizes distribution: (a) HIV dataset, (b) ToxCast dataset.}
\label{figure_molevules_hiv_toxcast}
\end{figure}

\clearpage

\section{Feature extraction pipeline visualization}
\label{appendix_feature_extraction_pipeline_visualization}

Here, we present a plot of feature extraction pipeline of MOLTOP, i.e. extracted features and their aggregation. We recall that $\mathrm{deg}$ means degree of a node, $DN$ is the multiset of degrees of neighbors, EBC is Edge Betweenness Centrality, ARI is Adjusted Rand Index, and SCAN is the SCAN Structural Similarity score. We aggregate topological features with histograms, resulting in integer features. Depending on the feature, we use either $11$ bins or the number of bins equal to the median size of the molecules in the training set. For each of the $90$ atom types (atomic numbers) and $5$ bond types, we compute the sum, mean and standard deviation in the molecule. All features are finally concatenated into the full MOLTOP feature vector.

\begin{figure}[ht!]
\centering
\begin{tabular}{cc}
\includegraphics[width=\columnwidth]{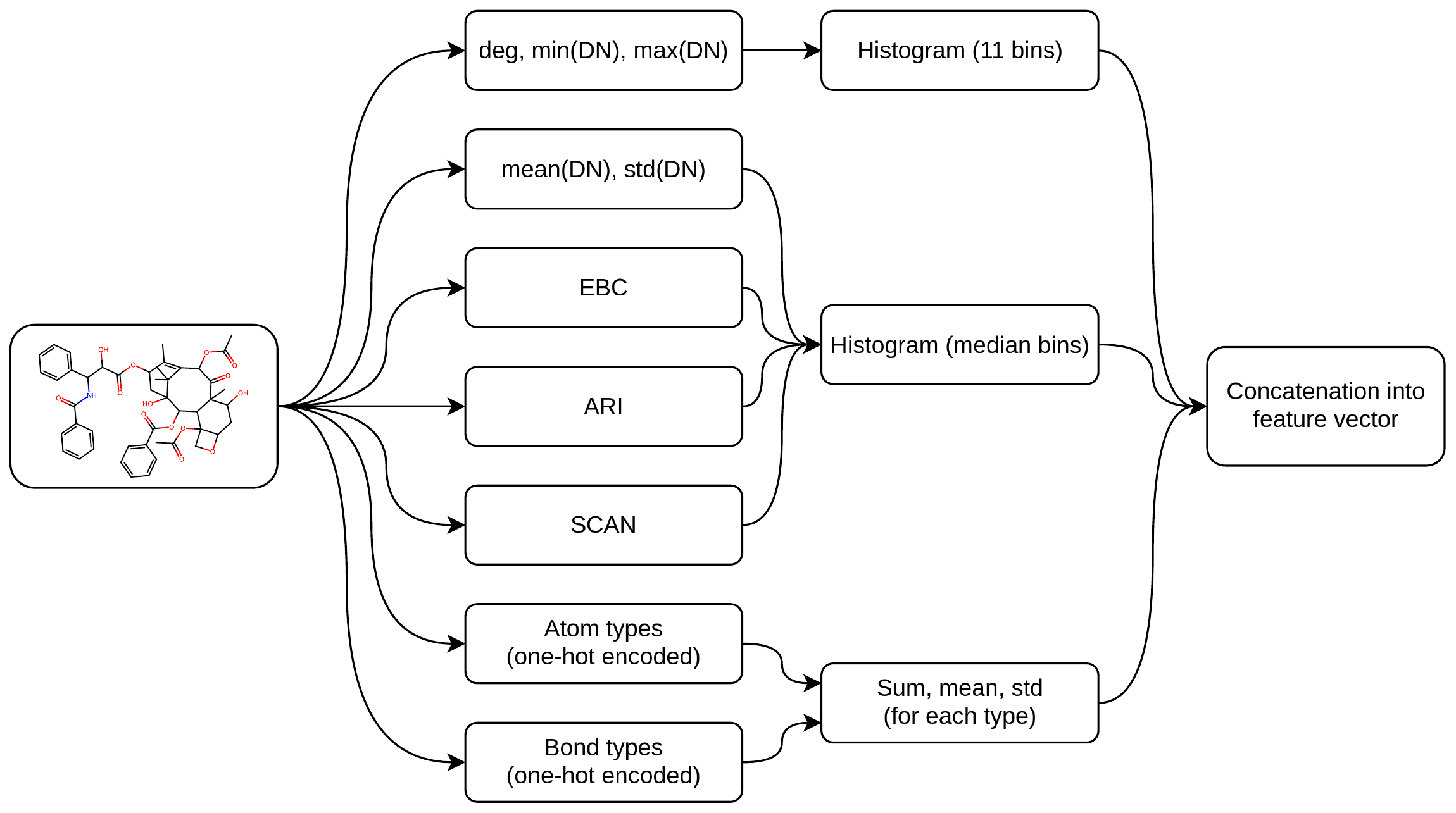}
\end{tabular}
\caption{Feature extraction scheme in MOLTOP.}
\label{figure_decaprismane_dodecahedrane}
\end{figure}

\clearpage

\section{Computational complexity of Adjusted Rand Index and SCAN scores}
\label{appendix_computational_complexity}

We present the derivation of the computational complexity for Adjusted Rand Index (ARI) and SCAN Structural Similarity scores for the edges in the graph. We denote the highest vertex degree as $k$. We assume the adjacency sets representation of a graph $G = (V, E)$, and amortized complexity of checking existence of element in a set as $O(1)$.

\begin{theorem}
The computational complexity for computing Adjusted Rand Index (ARI) for all existing edges in the graph is $O(k |E|)$, with the worst case complexity $O(|V||E|)$, which occurs for full graphs.
\end{theorem}

\begin{proof} 
The formula for computing ARI for a single edge $e = \{u, v\}$ is:
\begin{equation}
ARI(u, v) = \frac{2(ad - bc)}{(a + b)(b + d) + (a + c)(c + d)},
\end{equation}
It includes computing four values. $a$ is the number of edges to other vertices that $u$ and $v$ have in common, which reduces to set union, with complexity $O(k+k) = O(k)$; $b$ is the number of edges to other nodes for $u$ that $v$ does not have, which reduces to set intersection, with complexity $O(k)$; $c$ is the number of edges to other nodes for $v$ that $u$ does not have, and it also has complexity $O(k)$; $d$ is the number of edges to other nodes that neither $u$ nor $v$ has, and it reduces to computing the difference between total number of vertices $|V|$ and size of neighborhoods' union, which is already calculated for $a$, therefore it is simply the difference of two integers, with complexity $O(1)$. Total complexity for a single edge is, therefore, $O(k + k + k + 1) = O(k)$.

Computing ARI for all edges requires evaluating the expression above for $|E|$ edges. Therefore, the total complexity is $O(k |E|)$. For full graphs, for which all vertices have degree $|V|$, and therefore $k = |V|$, the complexity becomes $O(|V||E|)$.
\end{proof}

\begin{theorem}
The computational complexity for computing SCAN Structural Similarity scores for all existing edges in the graph has complexity $O(k |E|)$, with worst case complexity $O(|V||E|)$, which occurs for full graphs.
\end{theorem}

\begin{proof}
The formula for computing SCAN score for a single edge $e = \{u, v\}$ is:
\begin{equation}
SCAN(u, v) = \frac{|\mathcal{N}(u) \cap \mathcal{N}(v)| + 1}{\sqrt{(\mathrm{deg}(u) + 1)(\mathrm{deg}(v) + 1)}}
\end{equation}
Computing size of neighborhoods' intersection reduces to computing size of set intersection, which is $O(k)$. Degrees of vertices can be computed during the iteration needed for computing set intersection. This way, total complexity for a single edge is $O(k)$.

Computing SCAN scores for all edges requires evaluating the expression above for $|E|$ edges. Therefore, the total complexity is $O(k |E|)$. For full graphs, for which all vertices have degree $|V|$, and therefore $k = |V|$, the complexity becomes $O(|V||E|)$.
\end{proof}

We note that molecular graphs are very sparse, i.e. $|E| << |V|^2$. In particular, the number of bonds very rarely exceeds $10$, especially in medicinal chemistry. Therefore, we can treat $k$ as a constant, and this way the expected complexity reduces to $O(|E|)$ for this kind of graphs.

Alternatively, triangle counting algorithms can be used for computing neighborhood intersections. They typically have complexity $O(a(G) |E|)$, where $a(G)$ is the arboricity of the graph \cite{Triangle_counting_Ortmann}. This is particularly useful for molecular graphs, since almost all known molecules (with some exceptions, e.g. for crystals) are planar \cite{Molecules_planar_Simon}, and planar graphs have arboricity at most 3 \cite{Planar_graph_arboricity}. Utilizing this constant, the complexity reduces to $O(|E|)$, the same as for neighborhood intersection-based algorithms.

\clearpage

\section{Detailed results of model improvements}
\label{appendix_detailed_improvements}

Here, we present the more detailed results of proposed model improvements in \cref{table_appendix_detailed_improvement_results}. We report test AUROC, i.e. mean and standard deviation across $10$ runs, for all datasets.

\begin{table}[h!]
\caption{The results of proposed model improvements (extended version).}
\centering
\resizebox{\textwidth}{!}{
\begin{tabular}{|c|c|c|c|c|c|c|c|c|c|}
\hline
Model                                                                              & BACE         & BBBP         & HIV          & ClinTox      & MUV          & SIDER        & Tox21        & ToxCast      & \begin{tabular}[c]{@{}c@{}}Avg.\\ AUROC $\uparrow$\end{tabular} \\ \hline
LDP, 50 bins                                                                       & 80.5$\pm$0.3 & 63.3$\pm$0.4 & 72.1$\pm$0.4 & 50.0$\pm$0.9 & 58.2$\pm$2.1 & 59.8$\pm$0.5 & 66.7$\pm$0.2 & 59.5$\pm$0.4 & 63.8$\pm$0.7                                                    \\ \hline
\begin{tabular}[c]{@{}c@{}}Adding topological\\ edge features\end{tabular}                 & 81.3$\pm$0.4 & 65.0$\pm$0.5 & 74.9$\pm$0.8 & 52.4$\pm$2.5 & 61.3$\pm$1.4 & 59.8$\pm$0.5 & 68.7$\pm$0.3 & 60.3$\pm$0.5 & 65.6$\pm$0.9                                                    \\ \hline
\begin{tabular}[c]{@{}c@{}}Adding atoms and bonds\\ features\end{tabular}             & 81.3$\pm$0.4 & 68.4$\pm$0.6 & 78.5$\pm$0.6 & 54.3$\pm$1.9 & 65.9$\pm$1.9 & 64.7$\pm$0.5 & 75.4$\pm$0.4 & 63.5$\pm$0.3 & 69.0$\pm$0.8                                                    \\ \hline
Median bins                                                                        & 81.1$\pm$0.4 & 68.8$\pm$0.5 & 79.0$\pm$0.7 & 54.8$\pm$2.6 & 67.6$\pm$1.4 & 64.8$\pm$0.5 & 75.7$\pm$0.4 & 63.7$\pm$0.3 & 69.4$\pm$0.9                                                    \\ \hline
\begin{tabular}[c]{@{}c@{}}Reducing LDP bins, \\ dropping constant features\end{tabular} & 81.0$\pm$0.4 & 69.3$\pm$0.4 & 78.8$\pm$0.6 & 55.0$\pm$1.7 & 73.4$\pm$1.7 & 65.8$\pm$0.4 & 75.9$\pm$0.3 & 63.7$\pm$0.3 & 70.4$\pm$0.7                                                    \\ \hline
\begin{tabular}[c]{@{}c@{}}Tuning Random Forest\\ hyperparameters\end{tabular}       & 82.9$\pm$0.2 & 68.9$\pm$0.2 & 80.8$\pm$0.3 & 66.7$\pm$1.9 & 73.6$\pm$0.7 & 66.0$\pm$0.5 & 76.3$\pm$0.2 & 64.4$\pm$0.3 & 72.5$\pm$0.5                                                    \\ \hline
\end{tabular}
}
\label{table_appendix_detailed_improvement_results}
\end{table}

The last row corresponds to MOLTOP, with all improvements, and it achieves the best result in all cases. Additionally, the proposed improvements not only always result in the increase of average AUROC, but also almost always improve results for all datasets. Analyzing the detailed effects of particular changes on different datasets allows us to infer, which features are the most important for a given task.

For example, for BBBP dataset, introducing the atoms and bonds features gave the largest improvement of $3.4\%$, which aligns with chemical insight that particular elements and bonds are well correlated with the ability to penetrate the blood-brain barrier. HIV dataset shows similar behavior, with $3.6\%$ improvement. On the other hand, introducing those features for SIDER and Tox21 results in negligible change of $0.1\%$ and $0.3\%$, respectively. However, the proposed topological descriptors increase AUROC by $4.9\%$ and $6.7\%$ for those datasets, which highlights that drug side effects and various toxicity targets are more affected by the overall topology of the molecule.

\section{Reproducibility and hardware details}
\label{appendix_reproducibility}

To ensure the full reproducibility of our results, we used Poetry tool \cite{Poetry} to pin the exact version of all dependencies in the project, including transitive dependencies of directly used libraries. We distribute the resulting \texttt{poetry.lock} file, as well as \texttt{requirements.txt} file generated from it, along with our source code. This ensures the exact reproducibility of all results that is OS-agnostic and hardware-agnostic.

We conduct all experiments on CPU, since some operations on GPU are inherently nondeterministic, e.g. those related to processing sparse matrices in PyTorch Geometric. Due to efficiency of MOLTOP, the usage of GPU is also not necessary. All experiments were run on a machine with Intel Core i7-12700KF 3.61 GHz CPU and 32 GB RAM, running Windows 10 OS. We additionally ran the experiments on a second machine with Intel Core i7-10850H 2.70 GHz CPU and 32 GB RAM, running Linux Ubuntu 22.04 OS. The results were exactly the same in all cases.

\section{Evaluation protocols of other GNNs}
\label{appendix_evaluation_protocols_of_GNNs}

Here, we compare the evaluation protocol presented in this paper with alternatives found in the literature. In particular, we focus on the distinction between different types of splits, and the subtle differences between them, which render many direct comparisons unfeasible.

Random split, typically used in machine learning, just randomly (or, precisely, pseudorandomly, since we can set the random seed) selects the test set. It is interpolative in nature, i.e. the test set roughly follows the overall distribution of the data. This is not realistic for molecular property prediction, where we are often interested in novel compounds. Those tasks are extrapolative in nature, i.e. it is expected for future molecules to be structurally different from the existing ones. If time information is available, we can use a time split, like for PDBbind dataset in \cite{MoleculeNet_Wu}. However, this is almost never the case, and we use scaffold split instead, also proposed for evaluation of molecular classification in \cite{MoleculeNet_Wu}. It aims to take the least common groups of structurally similar molecules into the test set, which requires out-of-distribution generalization to achieve a good score. In many cases, this is a good approximation of a time split \cite{DMPNN_Yang}.

Firstly, we compute the Bemis-Murcko scaffold \cite{Bemis_Murcko_scaffold} for each molecule, and then we group molecules by their scaffolds. The subtle differences in the algorithm dividing them into training, validation and test sets determine practical aspects of evaluating classification accuracy. In fact, they are the major source of differences in scores observed in molecular property prediction literature.

As described in \cite{MoleculeNet_Wu}, we put the smallest groups of scaffolds in the test set, until we get the required size, and then we do the same for the validation set. All other scaffolds, which are the most common, constitute the training set. This is a fully deterministic setting, and was used in e.g. \cite{Pretraining_GNNs_Hu,GEM_Fang}. Splits provided by OGB \cite{OGB_Hu} also follow this protocol.

On the other hand, multiple works, such as D-MPNN \cite{DMPNN_Yang} and GROVER \cite{GROVER_Rong}, explicitly state that they compute scaffold splits multiple times, which indicates a non-deterministic process. This is indeed the case, since \cite{DMPNN_Yang} explicitly describe that they put any scaffold groups larger than half the test size into the training set, and then the remaining groups are put randomly into training, validation and test datasets. This randomness will very likely result in larger scaffold groups in validation and training sets than in the case of deterministic scaffold split. This setting is called \q{balanced scaffold split} in \cite{GNN_pretraining_benchmark_Sun}.

This distinction actually makes a very significant difference in scores, as analyzed in detail by \cite{GNN_pretraining_benchmark_Sun}. \q{Balanced scaffold split} achieves much higher results, often by $5\%$ or as much as $20\%$ on BBBP dataset, for multiple models. This is particularly problematic, as this difference is very subtle and not highlighted in the papers at all.

GROVER \cite{GROVER_Rong} mentions that they use three different random-seeded scaffold splits. Checking the official code \cite{GROVER_Github}, we found \q{balanced\_scaffold} in multiple places, confirming that the authors were aware of the difference between scaffold split and \q{balanced scaffold split}. This is additionally evidenced by comments in the code and function arguments. For this reason, we conclude that high scores in \cite{GROVER_Rong} are, at least in some part, the result of this choice.

As a consequence of this splitting differences, we cannot compare our results directly with the ones presented in D-MPNN or GROVER papers. The scores for both models, taken from GEM paper \cite{GEM_Fang}, which we use for the comparison, are lower and much more in line with results for deterministic scaffold split, as presented in \cite{GNN_pretraining_benchmark_Sun}. This is, again, the easiest to check with BBBP dataset, on which the difference is about $20\%$ just due to the splitting strategy.

We cannot compare to R-MAT \cite{R_MAT_Maziarka}, because they use scaffold split only for BBBP, and use nonstandard datasets apart from BBBP and ESOL. Additionally, they use random split for other datasets. However, we point that they recalculate GROVER results for BBBP using scaffold split, and get the result that aligns with the one in the GEM paper.

As for AttentiveFP \cite{AttentiveFP_Xiong}, the results seem particularly troubling. In the paper, the authors state that they use scaffold split for BBBP, BACE and HIV datasets, following \cite{MoleculeNet_Wu} (later papers generally use scaffold splits for all MoleculeNet datasets). However, checking the official code \cite{AttentiveFP_Github_page}, the word \q{scaffold} does not appear anywhere in the code, and verifying the code for those $3$ datasets, the random split is used in every case. Additionally, the difference in results between the original paper, and AttentiveFP results in the GEM paper would indicate that this is indeed the case. Because of this, we also do not compare directly to AttentiveFP results from the original paper, but rather from \cite{GEM_Fang}.

In conclusion, comparison to other papers for molecular property prediction in many cases requires very in-depth verification of both papers, their exact wording, and analyzing the official code. Of course, there is nothing wrong with alternative evaluation protocols and splitting procedures, but the due to differences in terminology this can result in misunderstanding the actual evaluation protocol used. This is an unfortunate situation, and it requires further investigation for other papers.

\section{Estimation of GEM computational cost}
\label{appendix_GEM_cost}

Here, we provide an estimation of GEM computational cost for pretraining. While the total cost of pretraining on 20 million molecules is not stated in the paper \cite{GEM_Fang}, the authors provide a link to the official code on GitHub \cite{GEM_Github_page}. There, they provide a small subset of 2000 molecules from the ZINC dataset for a demo, with a note \q{The demo data will take several hours to finish in a single V100 GPU card}.

We make a very conservative assumption, that \q{several hours} means 5 hours. The entire pretraining dataset is about 10000 times larger, so we get 50 thousand GPU hours. Assuming 250 NVidia V100 GPUs (to compare to GROVER, which also used 250 V100 GPUs), this gives us 200 hours, or slightly over 8 days.

\section{Ablation study}
\label{appendix_ablation_study}

Here, we present the results of the ablation study. We remove one group of features at a time from MOLTOP, and present results in \cref{table_appendix_ablation_study}. We include the original MOLTOP results in the first row for reference.

\begin{table}[ht!]
\caption{Results of ablation study, after removing different groups of features.}
\centering
\resizebox{\textwidth}{!}{
\begin{tabular}{|c|c|c|c|c|c|c|c|c|c|}
\hline
Model                                                                                & BACE         & BBBP         & HIV          & ClinTox      & MUV          & SIDER        & Tox21        & ToxCast      & \begin{tabular}[c]{@{}c@{}}Avg.\\ AUROC $\uparrow$\end{tabular} \\ \hline
MOLTOP                                                                               & 82.9$\pm$0.2 & 68.9$\pm$0.2 & 80.8$\pm$0.3 & 66.7$\pm$1.9 & 73.6$\pm$0.7 & 66.0$\pm$0.5 & 76.3$\pm$0.2 & 64.4$\pm$0.3 & 72.5$\pm$0.5                                                    \\ \hline
Removed LDP features                                                                 & 80.9$\pm$0.3 & 69.0$\pm$0.2 & 79.2$\pm$0.2 & 65.8$\pm$1.7 & 65.6$\pm$0.8 & 66.1$\pm$0.2 & 75.9$\pm$0.2 & 63.8$\pm$0.3 & 70.8$\pm$0.5                                                    \\ \hline
Median bins instead of reduced                                                       & 83.4$\pm$0.2 & 68.8$\pm$0.2 & 80.5$\pm$0.3 & 66.3$\pm$2.1 & 67.2$\pm$1.3 & 65.3$\pm$0.3 & 76.2$\pm$0.2 & 64.3$\pm$0.1 & 71.5$\pm$0.6                                                    \\ \hline
Removed topological features                                                         & 83.3$\pm$0.3 & 68.2$\pm$0.2 & 79.3$\pm$0.4 & 64.7$\pm$2.0 & 69.8$\pm$1.0 & 66.8$\pm$0.2 & 75.8$\pm$0.1 & 64.4$\pm$0.2 & 71.5$\pm$0.6                                                    \\ \hline
\begin{tabular}[c]{@{}c@{}}Removed atoms and \\ bonds features\end{tabular}          & 81.7$\pm$0.1 & 64.9$\pm$0.2 & 76.5$\pm$0.4 & 54.4$\pm$2.5 & 64.1$\pm$1.0 & 60.8$\pm$0.4 & 69.2$\pm$0.2 & 60.7$\pm$0.2 & 66.5$\pm$0.6                                                    \\ \hline
50 bins instead of median                                                            & 83.0$\pm$0.2 & 68.5$\pm$0.2 & 80.6$\pm$0.3 & 64.8$\pm$2.0 & 70.5$\pm$0.8 & 66.2$\pm$0.4 & 76.2$\pm$0.2 & 64.3$\pm$0.3 & 71.8$\pm$0.6                                                    \\ \hline
\begin{tabular}[c]{@{}c@{}}Remove dropping \\ constant features\end{tabular}         & 82.8$\pm$0.2 & 68.6$\pm$0.2 & 80.7$\pm$0.4 & 66.5$\pm$2.1 & 72.8$\pm$1.2 & 66.1$\pm$0.4 & 76.3$\pm$0.2 & 64.3$\pm$0.2 & 72.3$\pm$0.6                                                    \\ \hline
\begin{tabular}[c]{@{}c@{}}Unoptimized Random Forest \\ hyperparameters\end{tabular} & 81.0$\pm$0.4 & 69.3$\pm$0.4 & 78.8$\pm$0.6 & 55.9$\pm$1.5 & 73.4$\pm$1.6 & 65.4$\pm$0.4 & 76.1$\pm$0.3 & 63.8$\pm$0.5 & 70.5$\pm$0.7                                                    \\ \hline
Remove max neighbors degree                                                          & 82.9$\pm$0.1 & 68.5$\pm$0.2 & 80.4$\pm$0.3 & 66.5$\pm$1.8 & 73.9$\pm$1.0 & 66.1$\pm$0.3 & 76.2$\pm$0.2 & 64.2$\pm$0.3 & 72.3$\pm$0.5                                                    \\ \hline
\end{tabular}
}
\label{table_appendix_ablation_study}
\end{table}

Removing any part decreases the average AUROC, and often by a large margin, validating our modelling choices. Removing atoms and bonds features results in the largest drop, which is expected, and emphasizes the importance of incorporating those features for molecular graph classification. The smallest drop is for removal of constant features, but the main goal of this step was to remove obviously useless features and reduce computational and memory complexity, so this was expected. In general, our proposed method shows graceful degradation and still performs well, after removal of any feature. Also, the worst result here, after removal of atoms and bonds, is still better than LTP results. Since this leaves only our topological features, this indicates that our chemical intuitions for their choice specifically for molecular data were correct.

We additionally check what is the impact of removing the weakest feature, the maximal degree of neighbors. The average AUROC is a bit lower, showing that while this feature may not be as useful as others, it still positively impacts the discriminative ability of MOLTOP.

\section{Graph kernels experiments}

Here, we provide results of additional experiments with graph kernels. We selected the most widely used kernels, representing various approches: vertex histogram (VH), edge histogram (EH), graphlet kernel, propagation kernel, shortest paths kernel, Weisfeiler-Lehman (WL), and WL Optimal Assignment (WL-OA). For node labels, we use atomic numbers. We tune the inverse regularization strength parameter C, considering values $[10^{-3}, 10^{-2}, 10^{-1}, 1, 10, 10^2, 10^3]$. We compare results on the same datasets as other models, except for HIV, MUV and peptides-func, for which we got OOM errors due to their size. See Table \cref{table_appendix_graph_kernels} for results.

\begin{table}[]
\caption{Results of experiments with graph kernels. The best results are bolded.}
\centering
\begin{tabular}{|c|c|c|c|c|c|c|c|}
\hline
Method                  & BACE          & BBBP          & ClinTox       & SIDER         & Tox21         & Avg. AUROC $\uparrow$ & Avg. rank $\uparrow$ \\ \hline
Edge histogram kernel   & 66.1          & 52.9          & 55.5          & 45.5          & 63.2          & 56.6                  & 7.2                  \\ \hline
Graphlet kernel         & 70.3          & 58.2          & 39.1          & 47.3          & 57.4          & 54.5                  & 7.2                  \\ \hline
Propagation kernel      & 71.7          & 62.8          & 68.3          & 60.8          & 67.5          & 66.2                  & 4.2                  \\ \hline
Shortest paths kernel   & 76.5          & 65.3          & 59.7          & 62.4          & 64.7          & 65.7                  & 4                    \\ \hline
Vertex histogram kernel & 66.0          & 59.7          & 58.7          & 59.4          & 60.1          & 60.8                  & 6.4                  \\ \hline
WL kernel               & 83.8          & 67.6          & 57.3          & 63.3          & 73.8          & 69.2                  & 3.4                  \\ \hline
WL-OA kernel            & \textbf{85.8} & \textbf{69.1} & 59.2          & 65.8          & 74.6          & 70.9                  & 2                    \\ \hline
MOLTOP         & 82.9          & 68.9          & \textbf{73.6} & \textbf{66.0} & \textbf{76.3} & \textbf{73.5}         & \textbf{1.6}         \\ \hline
\end{tabular}
\label{table_appendix_graph_kernels}
\end{table}

MOLTOP comes up on top, getting the best results on 3 datasets and close second on BBBP. The slightly worse results on BACE and BBBP shows that they are very topology-centric, in line with conclusions from ablation study in \cref{appendix_ablation_study}.

\section{Additional expressivity experiments}
\label{appendix_expressivity_experiments}

Here, we provide additional examples and the results of experiments for expressivity of MOLTOP, i.e. its ability to distinguish non-isomorphic graphs. In particular, we show that MOLTOP is able to distinguish graphs on which $1$-WL test \cite{Weisfeiler_Lehman} fails. In this section, we consider only topological features, i.e. degree features, EBC, ARI and SCAN histograms. They form a vector of integers, since they are the counts in histogram bins, therefore we deem two graphs distinguished if they differ at any index. In all cases, to better understand where the expressiveness of MOLTOP comes from, we analyze the results for the features independently, and for the full feature vector. We note that degree features, based on LDP, are equivalent to a WL-test with 2 iterations \cite{LDP}, therefore we include them as a control, expecting negative result in all cases.

Firstly, we provide examples of pairs of graphs from the previous publications on which various WL tests fail. We treat each pair of graphs as a separate $2$-sample dataset, making the number of bins equal to the size of those graphs. We consider only the pairs of the same size, since distinguishing differently sized graphs would be trivial.

In \cref{figure_decalin_bicyclopentyl}, we show decalin and bicyclopentyl, example from \cite{GNN_expressiveness_Sato,GSN_Bouritsas}. Those molecules are not isomorphic nor regular, but cannot be distinguished by $1$-WL test, and, by extension, by all typical message-passing GNNs. However, MOLTOP can distinguish them, due to inclusion of EBC - this is the only one of four topological features that is able to do so. The reason is that bicyclopentyl includes a bridge, which has a very high EBC value, and it does not appear in decalin. This also follows our chemical intuition and motivation for including EBC.

To compare the MOLTOP features against raw shortest path information, we present the example from \cite{Graphormer_Ying}, in \cref{figure_1_wl_simple}. Those two graphs are not distinguishable by $1$-WL test, but can be distinguished by using the sets of shortest paths distances, and therefore by the Graphformer. Blue and red nodes have different sets of shortest paths distances in two graphs. All MOLTOP features except for degree features can also distinguish those graphs. EBC utilizes more information than just the lengths of shortest paths, and detects a bridge. ARI and SCAN analyze the neighborhood connectivity structure, and can distinguish regular grid (\cref{figure_1_wl_simple}a) from the two-communities structure (\cref{figure_1_wl_simple}b).

\begin{figure}[!htb]
  \centering
  \begin{tabular}{cc}
  \includegraphics[width=0.22\columnwidth]{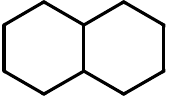} &
  \includegraphics[width=0.3\columnwidth]{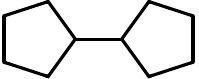} \\
  (a) & (b) 
  \end{tabular}
\vspace{0.2cm}
\caption{Two molecular graphs: (a) decalin, (b) bicyclopentyl.}
\label{figure_decalin_bicyclopentyl}
\vspace{0.2cm}
\end{figure}

\begin{figure}[!htb]
\centering
\includegraphics[width=0.75\textwidth]{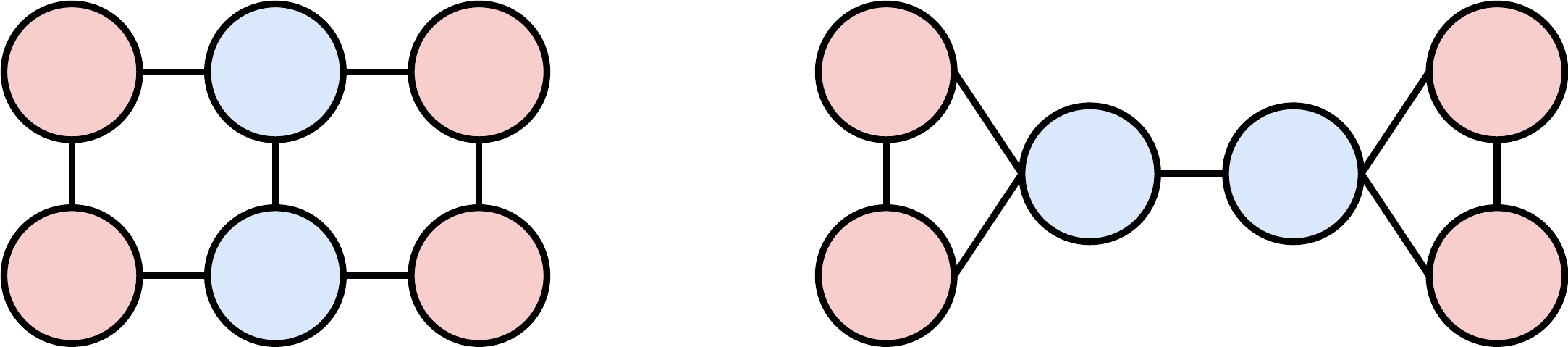}
\vspace{0.2cm}
\caption{The graphs, which cannot be distinguished by 1-WL-test, but their MOLTOP feature vectors are different.}
\label{figure_1_wl_simple}
\vspace{0.5cm}
\end{figure}

In the main body, we show MOLTOP achieves perfect result on \textit{sr25} dataset, which consists of strongly regular graphs. We also show the example of $3$-regular (not strongly regular) molecules from \cite{GNN_expressiveness_Sato} in \cref{figure_decaprismane_dodecahedrane}, decaprismane and dodecahedrane. They are not isomorphic, because decaprismane contains a $4$-cycle, while dodecahedrane does not. Those graphs are not distinguishable by typical message-passing GNNs, since they cannot distinguish $k$-regular graphs with the same size and features \cite{GNN_expressiveness_Sato}. All MOLTOP features, apart from degree features, can distinguish those graphs, most likely because $k$-regularity is a local feature and can be verified both by analyzing paths and neighborhood connectivity patterns.

\begin{figure}[ht!]
\centering
\begin{tabular}{cc}
\includegraphics[width=0.40\columnwidth]{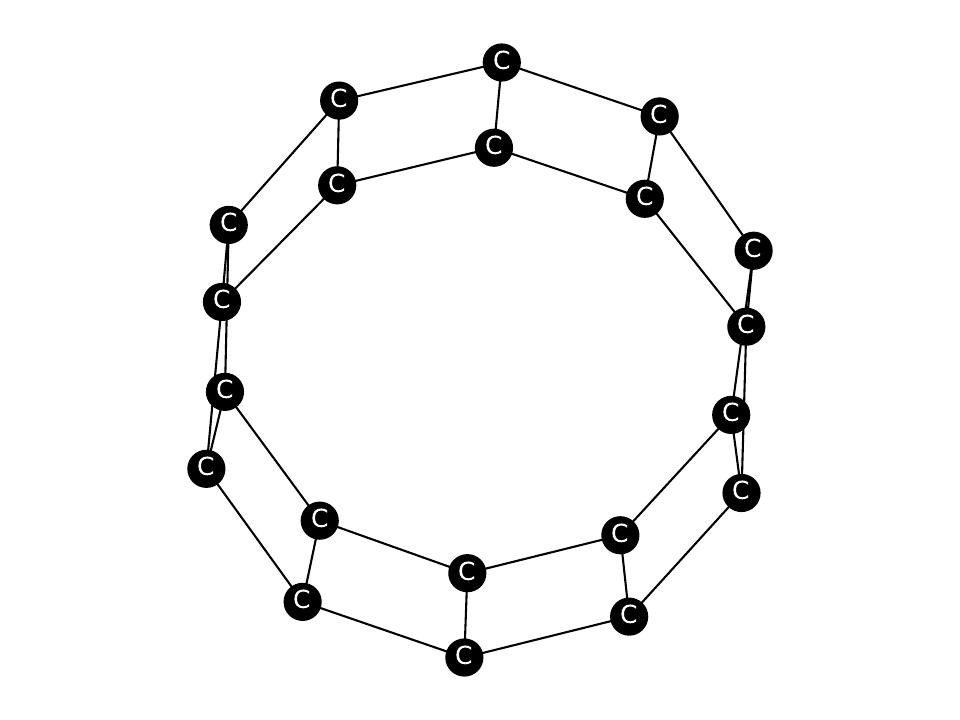}
&
\includegraphics[width=0.40\columnwidth]{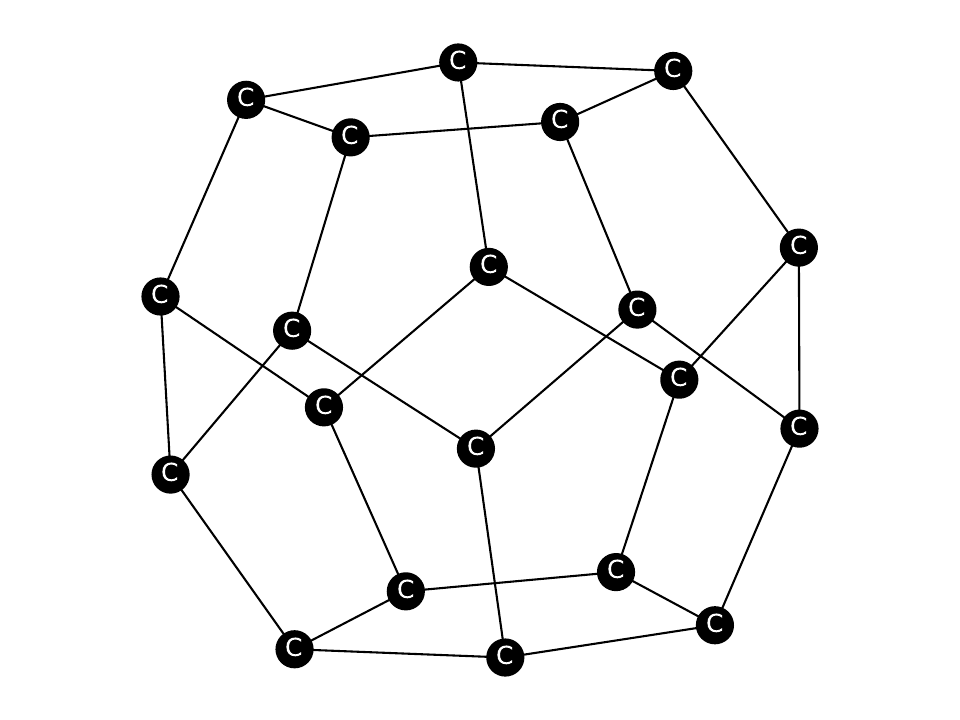}\\
(a) & (b)
\end{tabular}
\vspace{0.2cm}
\caption{Two $3$-regular graphs: (a) decaprismane, (b) dodecahedrane.}
\label{figure_decaprismane_dodecahedrane}
\vspace{0.5cm}
\end{figure}

There are, however, simple graphs which are not distinguishable by MOLTOP, as shown in \cref{figure_3_wl}, following example from \cite{GNN_expressiveness_Sato}. Those graphs are distinguishable by $3$-WL test, since it considers 3-tuple of vertices and can therefore detect disconnectedness in the left graph. MOLTOP, on the other hand, fails because it cannot detect this fact based on any of its features. However, we recall that $3$-WL cannot distinguish strongly regular graphs \cite{GNN_expressivity_testing_Balcilar}, while MOLTOP can, achieving perfect result on \textit{sr25}. This indicates that, interestingly, it does not fit into the traditional $k$-WL hierarchy.

\begin{figure}[h!]
\centering
\includegraphics[width=0.75\textwidth]{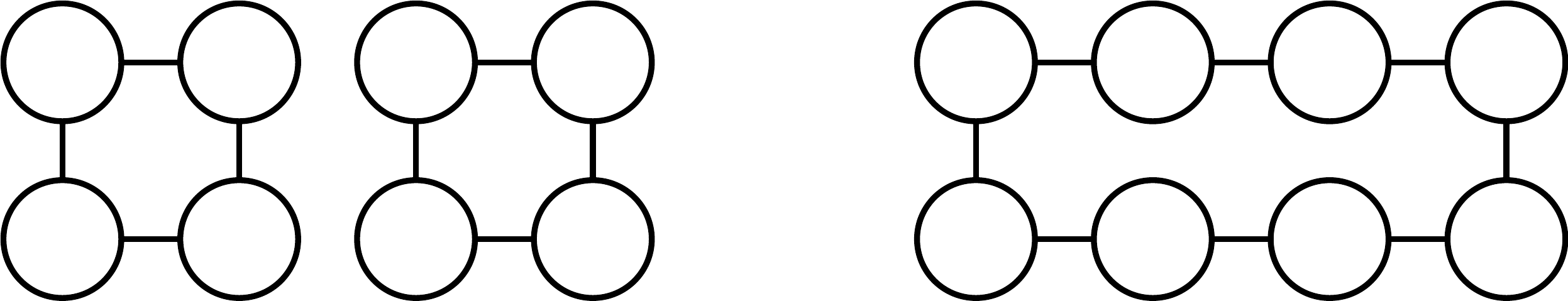}
\vspace{0.2cm}
\caption{Two graphs, which can be distinguished by $3$-WL test, but not by MOLTOP features.}
\label{figure_3_wl}
\end{figure}

\end{document}